%% file: main.tex
\documentclass{article}
\usepackage{arxiv}

\usepackage[utf8]{inputenc}
\usepackage[T1]{fontenc}
\usepackage{mathptmx}
\usepackage{amsmath,amssymb,amsthm,bm}
\usepackage{booktabs,tabularx,graphicx,microtype,url,xcolor}
\usepackage{enumitem,placeins}
\usepackage{algorithm,algpseudocode}
\usepackage[numbers,sort&compress]{natbib}
\usepackage{hyperref}
\usepackage[nameinlink,capitalise,noabbrev]{cleveref}
\hypersetup{colorlinks=true,linkcolor=[rgb]{0.09,0.30,0.45},citecolor=[rgb]{0.09,0.30,0.45},
 urlcolor=[rgb]{0.09,0.30,0.45},
 pdftitle={Phase-HDC: Replacing Optimizer History with Gradient Thresholds in Discrete Phase Learning},
 pdfauthor={Ahmed Nebli}}

\newcommand{\PHDC}{Phase-HDC}
\newcommand{\ind}{\mathbf{1}}
\newcommand{\ZK}{\mathbb{Z}_K}
\newcommand{\loss}{\mathcal{L}}
\newcommand{\eps}{\varepsilon}
\DeclareMathOperator{\sign}{sign}

\DeclareMathOperator{\rnd}{round}
\DeclareMathOperator{\clip}{clip}
\newtheorem{proposition}{Proposition}
\newtheorem{definition}{Definition}
\crefname{equation}{Eq.}{Eqs.}
\Crefname{equation}{Eq.}{Eqs.}
\setlist[itemize]{leftmargin=1.3em,itemsep=.3em,topsep=.3em}
\setlist[enumerate]{leftmargin=1.5em,itemsep=.35em,topsep=.3em}
\algrenewcommand\algorithmicrequire{\textbf{Input:}}
\algrenewcommand\algorithmicensure{\textbf{Output:}}

\title{Phase-HDC: Replacing Optimizer History with Gradient Thresholds in Discrete Phase Learning}
\renewcommand{\shorttitle}{Phase-HDC: Replacing Optimizer History with Gradient Thresholds}
\date{}
\author{Ahmed Nebli\\
 Mathalyse Research\\
 \texttt{ahmed.nebli@mathalyse.com}}

\begin{document}
\maketitle

\begin{abstract}
Training a compact model often needs far more memory than storing it, because the optimizer keeps its own records of past gradients. For a hyperdimensional classifier whose learned parameters are low-bit angles, which we call a \emph{phase memory}, these records take several times more memory than the model itself. We ask whether such a model can be trained while storing nothing but the model. The proposed method, Phase-HDC, turns each stored angle by at most one step per update, against the sign of its current gradient, and only when that gradient is large enough. We show that this simple rule is the exact solution of a first-order loss model in which every changed parameter pays a fixed cost. When everything except the update rule is held fixed, Phase-HDC matches the accuracy of Adam with 6-bit moments while storing three times less. Across eleven image, tabular, and text datasets, it stores 16--23$\times$ less than standard float32 Adam and 4--6$\times$ less than 8-bit Adam. The price is an average loss of about five accuracy points against float32 Adam, while Phase-HDC is more accurate than 8-bit Adam on six of the eleven datasets, including byte-level text prediction, where 8-bit Adam collapses. Instrumented training runs explain these outcomes. Once parameters must sit on a discrete grid, Adam's moments mainly decide whether a parameter moves at all, a decision that a threshold on the current gradient can make without memory, and coarse quantization of the moments breaks this decision for inputs that the data rarely contain. The storage savings are logical state rather than measured hardware memory. Code available at: \url{https://github.com/ahmednebli/Phase-HDC}
\end{abstract}


\section{Introduction}
\label{sec:intro}
Learning on a device with a small, fixed memory is constrained less by the size of the model than by the size of everything that must be kept while the model learns. Hyperdimensional computing (HDC) is attractive for such devices~\citep{kanerva2009}. An HDC classifier stores a table of high-dimensional vectors, the \emph{item memory}, that represent input positions and input values. It combines them with two cheap operations and compares the result with one vector per class. Early HDC systems used random item memories that required no training. Recent work shows that learning the item memory improves accuracy over random encoders~\citep{duan2022,kim2023,kim2024} and can reduce the number of dimensions that a task requires~\citep{dejonghe2026}.

Learning introduces a cost that the deployed model does not reveal. Adam~\citep{kingma2015}, a standard choice for training such encoders by backpropagation, keeps two running statistics for every parameter, and low-precision training usually adds a high-precision master copy of the parameters as well. These arrays persist from one minibatch to the next. We call the ratio between the persistent training state and the persistent model state the \emph{state overhead} $\Gamma$ (\cref{def:overhead}). For a model whose coordinates are 6-bit angles, Adam with float32 moments and a float32 master copy has $\Gamma=16$: a memory budget that comfortably holds the trained classifier is sixteen times too small to train it.

Existing responses to this cost keep the optimizer's history and make it smaller. Block-wise 8-bit optimizers quantize the moment arrays~\citep{dettmers2022}, and Adafactor replaces a matrix of second moments by row and column statistics~\citep{shazeer2018}. Optimizers for binary networks discard the master copy but still carry a momentum-like statistic~\citep{helwegen2019}. Work on quantized phase representations shows that an HDC algebra survives storing each coordinate as a small integer angle index~\citep{snyder2026}, yet it concerns a fixed representation, not the state required to learn one. To our knowledge, no prior method trains a learned discrete phase memory with $\Gamma=1$, and the accuracy price of doing so has not been measured.

This paper studies that operating point. We use a \emph{phase memory}: every position, value, and class coordinate is an integer $q\in\{0,\ldots,K-1\}$ that represents the angle $2\pi q/K$, with $K=2^k$ for a bit width $k$. A useful picture is a bank of rotary knobs, each with $K$ notches. A training step may turn a knob one notch clockwise, one notch counterclockwise, or leave it still, and turning past the last notch returns to the first. Because the knob has no position between notches, the update needs no hidden continuous value. The proposed method, \PHDC{}, reads the direction of each turn from the sign of the current minibatch gradient and turns a knob only when the gradient magnitude exceeds a threshold. It keeps no other state. \Cref{fig:method} shows the complete computation, rendered from a trained model. The rule is not a heuristic: it is the exact solution of a first-order loss model in which every changed coordinate pays a fixed write cost.

The question is then what this simplicity costs, and the answer depends on the competitor. We answer it with five lines of study of decreasing control (\cref{tab:evidence}). A controlled experiment on handwritten digits fixes the phase grid, classifier head, initialization, data split, and validation-only tuning, so that only the update rule and its stored history differ. Screening studies on the same digits place \PHDC{} among eight training systems, test an alternative item memory, and simulate fixed-point gradient arithmetic. An evaluation on six public image and tabular datasets measures the trade-off on harder tasks, and a character-level study on five text corpora extends it to next-symbol prediction with vocabularies of up to 256 symbols. Instrumented replays of the controlled runs record how often and how far each method moves each coordinate, which turns several questions about mechanisms into measurements. The picture that emerges has three parts, which are our contributions:
\begin{enumerate}
\item \textbf{A moment-free update with an exact local justification.} \PHDC{} trains a learned phase memory with $\Gamma=1$. Its thresholded one-notch step exactly minimizes a linearized loss with a per-coordinate write penalty (\cref{prop:local}), and the phase width couples the angular step and the effective threshold in a way we make explicit (\cref{eq:precision}).
\item \textbf{A measured accuracy and memory trade-off.} With identical 6-bit parameters and heads, \PHDC{} matches the mean accuracy of projected Adam with 6-bit moments at one third of its state. Across eleven image, tabular, and text datasets, it keeps $16$ to $23\times$ less state than Adam with continuous float32 parameters, at an average cost of $5.3$ accuracy points ($7.0$ on the image and tabular datasets, $3.3$ on text). It keeps $4$ to $6\times$ less state than 8-bit Adam and is more accurate on six of the eleven datasets, including byte-level enwik8, where 8-bit Adam collapses to $0.8\%$.
\item \textbf{A mechanistic account of what optimizer state buys.} On a phase grid, Adam's moments act mainly as a history-based gate that decides whether a coordinate moves; small learning rates close it completely, and per-tensor moment quantization breaks it for rarely observed input values (Spearman rank correlation $-0.98$ between visit frequency and large-jump rate), a mechanism consistent with the collapses of quantized Adam on images and on byte-level text. The accuracy advantage of float32 Adam comes mainly from its continuous parameters, not its moments. The phase width matters through the update as much as through the representation, and the update tolerates 4-bit gradient arithmetic in simulation.
\end{enumerate}

\begin{figure}[t]
\centering
\includegraphics[width=\textwidth]{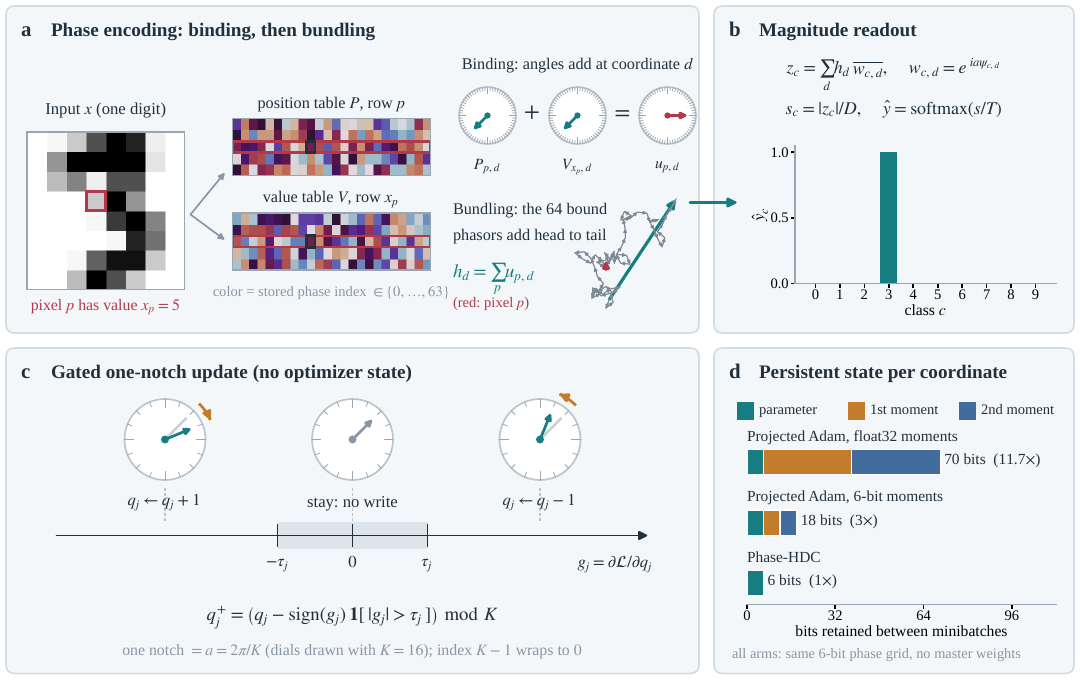}
\caption{\textbf{Overview of \PHDC{}, rendered from the trained digit model (seed 1).} (a)~An input digit selects one row of the learned position table $P$ for each pixel $p$ and one row of the shared value table $V$ for its value $x_p$. Colors show stored phase indices. At each coordinate $d$, binding adds the two angles; bundling then sums the resulting unit phasors over all 64 pixels head to tail, giving the complex coordinate $h_d$ (teal). (b)~The bundle is compared with unit-modulus class vectors, and the magnitude of each overlap enters a softmax. The bars are the trained model's class probabilities for this input. (c)~Training moves a stored index by at most one notch per step: against the gradient sign when $|g_j|>\tau_j$, and not at all inside the dead zone. Dials are drawn with $K=16$ so that one notch is visible; the model uses $K=64$. (d)~Bits retained per coordinate between minibatches for methods that share the same 6-bit phase parameters. \PHDC{} keeps only the parameter. Values are logical array state, excluding activations, gradients, and quantizer metadata.}
\label{fig:method}
\end{figure}

The paper is organized as follows. \Cref{sec:background} introduces phase vectors and optimizer state for readers new to either. \Cref{sec:related} positions the work, \cref{sec:method} derives the method, \cref{sec:setup,sec:results} describe the studies and their results, \cref{sec:discussion} explains the mechanisms behind them, and \cref{sec:limitations} states the main limitations with the work that would remove them.

\input{background}
\input{related}
\input{method}
\input{experiments}
\input{discussion}

\section{Conclusion}
\label{sec:conclusion}
A learned phase memory can be trained without any optimizer history. \PHDC{} turns each stored phase by at most one notch, in the direction opposite to the current gradient and only when the gradient is large enough, and this rule is the exact minimizer of a write-penalized first-order model. The results describe a clear trade-off. When parameters, head, and tuning protocol are identical, removing Adam's moments costs no measurable accuracy and saves two thirds of the persistent state. Against Adam with continuous float32 parameters, the same model keeps 16 to 23 times less state across eleven image, tabular, and text datasets, and loses 5.3 accuracy points on average. Against 8-bit Adam it keeps 4 to 6 times less state and is more accurate on six of the eleven datasets. At four bits the saving on digits reaches 26.6 times, but the accuracy cost grows to 11.7 points.

The mechanisms behind these numbers are specific. On a discrete grid, Adam's moments serve mainly to decide which coordinates move; a magnitude threshold on the current gradient makes a comparable decision without memory but loses Adam's invariance to gradient scale. The accuracy that float32 Adam retains comes mainly from its continuous parameters. Coarse per-tensor moment quantization fails on parameter rows that the data visit rarely, and the phase width acts on the update as strongly as on the representation. Moment-free training is therefore an operating point on the accuracy and memory frontier rather than a replacement for adaptive optimization. Whether that point is useful in practice depends on a packed integer implementation that converts logical state savings into measured memory, and on controlled comparisons against well-engineered compressed optimizers at larger scale.

\bibliographystyle{plainnat}
\bibliography{references}

\clearpage
\appendix
\input{appendix}
\end{document}

%% file: background.tex
\section{Background: Phase Vectors and Optimizer State}
\label{sec:background}
This section introduces the two ingredients of the paper for readers outside hyperdimensional computing and optimization. The first is a classifier whose parameters are angles. The second is the memory that an optimizer must carry while it trains such a classifier. Readers familiar with both can continue to \cref{sec:related}.

\subsection{Hyperdimensional Classification with Phase Vectors}
\label{sec:bg-hdc}
A hyperdimensional classifier represents everything it knows as long vectors of the same length $D$, typically hundreds or thousands of coordinates~\citep{kanerva2009}. Classification proceeds in three steps. First, the classifier looks up one vector for each part of the input: for an image, one vector for every pixel position and one vector for every possible pixel value. The table that holds these vectors is the \emph{item memory}. Second, it combines the looked-up vectors into a single vector that describes the whole input. Third, it compares this description with one stored vector per class and predicts the most similar class.

Two operations do the combining. \emph{Binding} attaches a value to a position, for example ``value 5 at pixel 27''. It must produce a vector that is distinct from both of its inputs, so that ``5 at pixel 27'' and ``5 at pixel 28'' are not confused. \emph{Bundling} superposes all bound pairs of an input into one vector that remains similar to each of them, much as a chord remains similar to each of its notes. With these two operations a fixed-size vector can describe an input of any size.

In a phase representation, every coordinate is an angle, which we picture as the hand of a clock. Binding adds angles coordinate by coordinate, and bundling adds the corresponding unit-length arrows head to tail. A small example makes this concrete. Suppose angles are stored with eight notches, so one notch is $45^\circ$. If the position vector of pixel 27 holds notch 2 ($90^\circ$) at some coordinate and the value vector of value 5 holds notch 3 ($135^\circ$) there, then their binding holds notch 5 ($225^\circ$) at that coordinate. Bundling then adds the arrows of all pixels at that coordinate. If many pixels point in similar directions, the sum is long; if they point in all directions, they largely cancel. The class score measures how well this bundle lines up with a class vector. \Cref{fig:method}a shows a real bundle of 64 pixels.

Phase representations are attractive for small devices because an angle with $K=2^k$ notches costs only $k$ bits, binding is integer addition modulo $K$, and the cosine and sine of $K$ angles fit in a small lookup table~\citep{snyder2026}. A 6-bit phase, for instance, distinguishes 64 directions. When the item memory and the class vectors are learned rather than random, each stored angle becomes a trainable parameter.

\subsection{Persistent Training State and Its Sources}
\label{sec:bg-state}
Training adjusts parameters in small steps computed from a minibatch gradient. Between two minibatches, the training procedure must keep certain arrays in memory. We call these arrays the \emph{persistent state}. Activations and gradients are not part of it, because they are recomputed for every minibatch and can be discarded afterwards. Three kinds of arrays typically persist.

The first is the parameters themselves. The second is a \emph{latent} or \emph{master} copy of the parameters in higher precision. Low-precision networks usually keep one, because a small update rounded to a coarse grid is erased, whereas a high-precision copy can accumulate many small updates until they cross a grid point~\citep{gupta2015,helwegen2019}. The third is the optimizer's memory of past gradients. Adam~\citep{kingma2015} keeps two such arrays. Its first moment is a running average of recent gradients and acts like the velocity of a rolling ball, carrying the update in a consistent direction. Its second moment is a running average of squared gradients and acts like a per-parameter volume control, shrinking steps where gradients are large or erratic and enlarging them where gradients are small.

\begin{definition}[State overhead]
\label{def:overhead}
The state overhead $\Gamma$ of a training method is the ratio between its persistent state and the persistent state of the trained model it produces.
\end{definition}

The overhead can be large. Consider a model of $10^6$ coordinates stored as 6-bit phases, which occupies 0.75~MB. Training it with Adam, float32 moments, and a float32 master copy keeps 96 bits per coordinate, 12~MB in total, so $\Gamma=16$. Without the master copy but with float32 moments, $\Gamma=70/6\approx11.7$. With 6-bit moments, $\Gamma=3$. Each of these configurations still stores optimizer history. The method of this paper stores none, so $\Gamma=1$: a device that can hold the trained model can hold everything needed to continue training it.

The removal of the master copy has a consequence that shapes the whole paper. Without it, every update must land exactly on a notch, so an optimizer can only turn a parameter by a whole number of notches or leave it unchanged. An update rule for such parameters therefore makes two decisions at every step: whether to move a coordinate, and in which direction. \Cref{sec:method} derives a rule that makes both decisions from the current gradient alone, and \cref{sec:discussion-gating} shows that Adam, when forced onto the same notches, makes the same two decisions but bases the first one on its stored history.

%% file: related.tex
\section{Related Work}
\label{sec:related}
The question studied here sits between three lines of work. Phase-based vector-symbolic representations supply the model. Trainable HDC encoders show that learning the model is worthwhile. Memory-efficient and sign-based optimizers address the state that learning requires. We discuss each line with respect to one question: how much persistent state does it need beyond the representation itself, and why?

\subsection{Phase-Based Vector-Symbolic Representations}
HDC and vector-symbolic architectures represent structured objects by distributed vectors and manipulate them with a small algebra~\citep{kanerva2009}. Binding associates two vectors, for example a position with the value observed there, and produces a vector dissimilar to both. Bundling superposes several vectors into one that remains similar to each of them. Holographic reduced representations implement binding by circular convolution of real vectors~\citep{plate1995}. In the Fourier domain, the same operation becomes coordinate-wise multiplication of unit complex numbers, that is, addition of their angles. This phase view underlies the model in \cref{sec:method}: binding adds two stored angles, and bundling adds the resulting complex numbers.

Fractional binding extends the phase algebra to continuous quantities. Spatial semantic pointers encode a coordinate by raising a fixed base vector to a real power, so that nearby locations receive similar codes~\citep{komer2019}. Such codes tie all values of an attribute to one frequency vector. Our model does not impose this tie. It learns a separate phase vector for every position and every value, which gives the optimizer more freedom and makes each table entry an independent parameter.

The precision of phase codes has been studied directly. qFHRR stores each coordinate as a discrete phase index and implements binding, unbinding, bundling, and similarity with modular integer arithmetic and lookup tables, showing that the algebra retains its fidelity at low bit widths~\citep{snyder2026}. This result is the representational precedent for our phase memory. It leaves open the problem addressed here: a quantized phase code that is \emph{learned} by gradient descent needs an update rule, and that rule determines how much state must accompany the code during training.

Complex-valued models with magnitude-based readouts also appear outside HDC. A sequence model by the author evolves a complex state with unitary dynamics and reads it out through squared amplitudes, as in the Born rule of quantum mechanics~\citep{nebli2026}. A subsequent study, \emph{The Quantum Shortcut}, evaluates complex phase-state models with state-space and attention backbones and reports fewer optimization steps to matched validation losses, with the attention advantage confined to early training~\citep{nebli2026shortcut}. That work concerns sequence-model dynamics; our focus is the persistent state required to train discrete phase memories. The present classifier shares only the use of complex magnitudes. It has no recurrence, its class score is the modulus of an overlap rather than a probability, and that score enters an ordinary softmax.

\subsection{Trainable Hyperdimensional Encoders}
Fixed random item memories make HDC training cheap but inefficient: a random encoder needs many dimensions to separate classes that a learned encoder separates with few. Several methods therefore learn the encoder. ManiHD adapts the encoding through a trainable manifold construction~\citep{zou2021}. LeHDC casts HDC classification as a binary neural network and trains it with a principled learning procedure~\citep{duan2022}. TrainableHD propagates classification error into the encoding vectors and combines training with quantization-aware inference~\citep{kim2023}; its extension adds adaptive optimization to this pipeline~\citep{kim2024}. THDC learns embeddings and a binary classification layer jointly by backpropagation and reports large reductions in the dimension needed on image benchmarks~\citep{dejonghe2026}.

These methods establish that learned item memories are accurate, and we do not claim learning itself as a contribution. Their training procedures, however, rely on continuous latent parameters, adaptive moment estimates, or both. Their reported memory figures describe the deployed model. The persistent state needed while training is larger and is typically not reported. For an edge device that must learn in place, this unreported state determines feasibility. Our experiments use a different encoder and readout from these systems, so their published accuracies are context rather than baselines.

\subsection{Compressed Optimizer State and Low-Precision Training}
A direct response to optimizer memory is to keep the adaptive rule and compress its statistics. Dettmers et al.\ quantize Adam's moments to 8 bits with block-wise dynamic quantization, in which every small block of a tensor receives its own range, and add mechanisms for stability~\citep{dettmers2022}. Block-wise ranges matter because moment tensors span many orders of magnitude, and a single range for a whole tensor wastes its levels on a few extreme entries. Our quantized-Adam baselines use the simpler per-tensor affine construction. \Cref{sec:discussion-quantization} shows that this choice causes a specific, measurable failure, so the baselines characterize that construction rather than 8-bit optimization in general. Adafactor reduces state differently: it approximates a matrix of second moments by the outer product of row and column statistics~\citep{shazeer2018}. Both approaches lower the state overhead while preserving some per-parameter history, so neither reaches $\Gamma=1$.

Low-precision training research also shows that rounding a small update onto a coarse grid can erase it. Gupta et al.\ find that fixed-point training succeeds at low precision only with stochastic rounding, because round-to-nearest discards every update smaller than half a quantization step~\citep{gupta2015}. This effect is central to the projected-Adam baseline in \cref{sec:setup}, which rounds every candidate step onto the phase grid without a continuous accumulator. WAGE goes further and discretizes weights, activations, gradients, and errors to train with integers end to end~\citep{wu2018}. Our study does not implement such a pipeline. It computes gradients in floating point and counts the persistent state that an integer implementation would have to keep.

\subsection{Sign-Based and Latent-Free Discrete Optimization}
Sign-based methods supply the update direction we use. signSGD updates each parameter by the sign of its stochastic gradient and analyzes when this compressed direction makes progress, in terms of gradient noise and curvature~\citep{bernstein2018}. Balles and Hennig decompose Adam itself into a sign component and a variance-dependent magnitude, interpreting the ratio of first moment to root second moment as a per-coordinate signal-to-noise estimate~\citep{balles2018}. This decomposition becomes concrete on a discrete grid: \cref{sec:discussion-gating} shows that projected Adam's moments decide mainly \emph{whether} a coordinate moves, which is precisely the role of the threshold in \PHDC{}.

Methods for binary networks show that discrete weights can be optimized without continuous latent copies. Bop flips a binary weight when an exponential moving average of its gradient exceeds a threshold, and argues that latent real-valued weights are better understood as inertia than as parameters~\citep{helwegen2019}. Bop removes the latent weight but keeps the moving average, so its state overhead remains above one. Lion also applies a sign update and keeps one momentum array~\citep{chen2023}. ThermoLion, by the author, blends sign and magnitude updates according to an estimate of gradient reliability that it maintains across steps~\citep{nebli2025}. These optimizers are the closest competitors in spirit, and each retains at least one array of per-parameter history.

\PHDC{} differs from all of them in what it stores and in the domain it acts on. It thresholds the instantaneous minibatch gradient and stores nothing else. Its parameters live on a circle, so a step that leaves the last index re-enters at the first, and the step has a fixed angular size set by the bit width. Its threshold has an exact interpretation as a write cost. Neither the sign direction nor latent-free discrete optimization is new; the contribution is their combination on learned circular coordinates, the analysis of that combination, and a controlled measurement of what removing all optimizer history costs.

%% file: method.tex
\section{Phase-HDC: A Moment-Free Update for Phase Memories}
\label{sec:method}
This section formalizes the phase classifier of \cref{sec:bg-hdc}, derives the update rule, and states how the bit width and the persistent state enter the method. \Cref{tab:notation} in \cref{app:notation} collects all symbols.

\subsection{Phase Memory Classifier}
\label{sec:model}
An input is a vector of discretized feature values $x=(x_1,\ldots,x_F)$ with $x_p\in\{0,\ldots,B-1\}$, where $F$ is the number of positions and $B$ the number of value bins. Let $D$ denote the representation dimension, $C$ the number of classes, $k\geq2$ the phase bit width, $K=2^k$ the number of phase levels, and $a=2\pi/K$ the angle of one notch. We write $\ZK=\{0,\ldots,K-1\}$ with addition modulo $K$.

\begin{definition}[Phase memory]
A phase memory is a triple of integer tables: a position table $P\in\ZK^{F\times D}$, a value table $V\in\ZK^{B\times D}$ shared by all positions, and a class table $\psi\in\ZK^{C\times D}$. Entry $q$ of any table represents the angle $aq$.
\end{definition}

The bound phasor of position $p$ at coordinate $d$ and the bundled input representation are
\begin{equation}
 u_{p,d}(x)=\exp\!\bigl[i a\,(P_{p,d}+V_{x_p,d})\bigr],
 \qquad h_d(x)=\sum_{p=1}^{F}u_{p,d}(x),
 \qquad d=1,\ldots,D.
 \label{eq:encoding}
\end{equation}
Binding is the addition of two stored angles inside the exponential, so no explicit modular reduction is needed in the forward map. Bundling sums complex numbers \emph{after} the exponential. The order matters: summing all angles first and exponentiating once would define a different and much less expressive encoder. \Cref{fig:method}a shows a bundle as the head-to-tail sum of the $F$ unit phasors at one coordinate.

Class $c$ is represented by the unit-modulus vector $w_c\in\mathbb{C}^D$ with $w_{c,d}=\exp(i a\psi_{c,d})$. The classifier scores and class probabilities are
\begin{equation}
 z_c=\sum_{d=1}^{D}h_d\,\overline{w_{c,d}},\qquad
 s_c=\frac{\sqrt{|z_c|^2+\eps}}{D},\qquad
 \hat y_c=\frac{\exp(s_c/T)}{\sum_{c'=1}^{C}\exp(s_{c'}/T)},
 \label{eq:readout}
\end{equation}
with $\eps=10^{-12}$ and temperature $T=0.3$. Training minimizes the mean cross-entropy $\loss=-\frac{1}{|\mathcal{B}|}\sum_{(x,y)\in\mathcal{B}}\log\hat y_y(x)$ over a minibatch $\mathcal{B}$ of labeled examples. The score $s_c$ measures how strongly the bundle aligns with class $c$ regardless of a global phase offset. It is neither a cosine similarity nor a normalized overlap, because $h$ does not have unit modulus and its scale is retained.

The integer tables define a differentiable function once each index is read as a real number. The extension of $P$, $V$, and $\psi$ to real arguments in \cref{eq:encoding,eq:readout} gives a smooth loss that is periodic with period $K$ in every coordinate. Its gradient can be evaluated at any integer state. The update below uses this gradient only to choose among integer moves, so it never differentiates through a rounding operation and never stores a continuous copy of the parameters.

\subsection{Write-Penalized One-Notch Update}
\label{sec:update}
Concatenate the three tables into one vector $q\in\ZK^N$ with $N=D(F+B+C)$, and let $g=\nabla_q\loss(q)\in\mathbb{R}^N$ be the gradient of the continuous extension on the current minibatch. A one-notch move changes coordinate $j$ by $\delta_j\in\{-1,0,1\}$. We assign a nonnegative cost $\tau_j$ to every coordinate that changes, and choose the move that minimizes the first-order change in loss plus the total write cost:
\begin{equation}
 \min_{\delta\in\{-1,0,1\}^N}\;
 g^\top\delta+\sum_{j=1}^{N}\tau_j\,\ind[\delta_j\neq0].
 \label{eq:local}
\end{equation}
The write cost is a modeling device that expresses a preference for leaving a coordinate unchanged unless the predicted gain justifies the change. It is not a measured energy.

\begin{proposition}[Exact local update]
\label{prop:local}
With ties resolved in favor of no change, a minimizer of \cref{eq:local} is
\begin{equation}
 \delta_j^\star=-\sign(g_j)\,\ind\bigl[\,|g_j|>\tau_j\,\bigr],\qquad
 q_j^{+}=\bigl(q_j+\delta_j^\star\bigr)\bmod K,\qquad j=1,\ldots,N.
 \label{eq:update}
\end{equation}
\end{proposition}
\begin{proof}
The objective is a sum of independent terms $g_j\delta_j+\tau_j\ind[\delta_j\neq0]$, so it is minimized coordinate by coordinate. The cost of leaving coordinate $j$ unchanged is zero. The two nonzero moves cost $\tau_j+g_j$ and $\tau_j-g_j$, and the smaller of the two is $\tau_j-|g_j|$, attained by the move $-\sign(g_j)$. This move is strictly better than staying exactly when $|g_j|>\tau_j$. Together, the coordinate-wise minimizers form a global minimizer of \cref{eq:local}.
\end{proof}

A numerical example shows the rule at work. With $\tau_j=3\times10^{-5}$, a coordinate whose gradient is $-2\times10^{-4}$ turns one notch forward, because moving forward is predicted to lower the loss by more than the write cost; a coordinate whose gradient is $10^{-5}$ stays where it is, because the predicted gain does not pay for the write. We use one threshold $\tau_{\mathrm{enc}}$ for the encoder tables $P$ and $V$ and one threshold $\tau_{\mathrm{cls}}$ for the class table $\psi$. Algorithm~\ref{alg:phasehdc} states the resulting training step. Three properties of \cref{prop:local} deserve emphasis. First, with $\tau_j=0$ the rule reduces to the familiar sign step, which solves the linearized problem over an $\ell_\infty$ ball restricted to integer moves. Second, a threshold is not a budget: the number of coordinates that change is decided by the gradient, not fixed in advance. Under the alternative constraint $\|\delta\|_0\leq R$, the solution instead moves the $R$ coordinates with the largest nonzero $|g_j|$; that budgeted variant is not evaluated here. Third, the rule depends only on the current gradient, so the persistent state between minibatches is the phase memory itself.

\begin{algorithm}[t]
\caption{One \PHDC{} training step}
\label{alg:phasehdc}
\begin{algorithmic}[1]
\Require phase memory $q=(P,V,\psi)\in\ZK^N$, minibatch $\mathcal{B}$, thresholds $\tau_{\mathrm{enc}},\tau_{\mathrm{cls}}$
\State Compute $\hat y(x)$ for $x\in\mathcal{B}$ with \cref{eq:encoding,eq:readout} \Comment{floating-point forward pass}
\State $g\gets\nabla_q\loss(q)$ of the continuous extension \Comment{floating-point backward pass}
\For{$j=1,\ldots,N$}
 \State $\tau_j\gets\tau_{\mathrm{cls}}$ if $j$ indexes $\psi$, else $\tau_{\mathrm{enc}}$
 \If{$|g_j|>\tau_j$}
  \State $q_j\gets(q_j-\sign(g_j))\bmod K$ \Comment{one notch against the gradient}
 \EndIf
\EndFor
\Ensure updated $q$; no other quantity is retained for the next step
\end{algorithmic}
\end{algorithm}

The proposition concerns a first-order model of one minibatch loss. To state its reach precisely, suppose the continuous extension has a $\Lambda$-Lipschitz gradient along the segment from $q$ to $q+\delta^\star$, and let $\mathcal{A}=\{j:|g_j|>\tau_j\}$ be the set of active coordinates. Since $\|\delta^\star\|_2^2=|\mathcal{A}|$, the standard descent lemma gives
\begin{equation}
 \loss(q+\delta^\star)-\loss(q)\;\leq\;-\sum_{j\in\mathcal{A}}|g_j|+\frac{\Lambda}{2}\,|\mathcal{A}|,
 \label{eq:descent}
\end{equation}
where periodicity identifies the endpoint with its representative modulo $K$. The minibatch loss decreases whenever the active gradient mass exceeds the curvature term. Every active coordinate satisfies $|g_j|>\tau_j$, so a threshold above $\Lambda/2$ would guarantee this condition. We do not estimate $\Lambda$ or enforce this condition, and the bound concerns one minibatch, not the population loss. \Cref{prop:local} therefore justifies the form of the step without implying monotone training or convergence.

\subsection{Coupling of Phase Width, Step Size, and Gate}
\label{sec:precision}
The bit width $k$ fixes both the alphabet of the representation and the geometry of its update. Let $\theta=aq$ denote a coordinate measured in radians, and write $g_q=\partial\loss/\partial q$ and $g_\theta=\partial\loss/\partial\theta$ for one coordinate. The chain rule gives
\begin{equation}
 g_q=a\,g_\theta,\qquad
 |g_q|>\tau\;\Longleftrightarrow\;|g_\theta|>\tau/a,\qquad
 |\Delta\theta|=a\ \text{for every active coordinate}.
 \label{eq:precision}
\end{equation}
Raising $k$ from 6 to 8 therefore has two effects at once. Each step becomes four times smaller in angle, and a fixed index-space threshold becomes four times stricter in angular units. In the knob picture, adding notches makes every turn finer and also makes the knob harder to turn. A threshold rescaled as $\tau=a\tau_\theta$ keeps the angular gate fixed at a given state but still changes the step size. A comparison across bit widths that uses one threshold grid, as the sweep in \cref{sec:results-width} does, therefore varies the representation, the step, and the gate together. \Cref{sec:discussion-width} returns to this coupling and adds a third effect that is specific to shared tables.

\subsection{Persistent-State Accounting}
\label{sec:accounting}
We count the ideal packed bits held in parameter and per-coordinate optimizer arrays between minibatches, which is the persistent state of \cref{sec:bg-state}, and express it through the state overhead $\Gamma$ of \cref{def:overhead}.

For \PHDC{} with $N$ coordinates of $k$ bits,
\begin{equation}
 M_{\mathrm{phase}}=\frac{kN}{8}\ \text{bytes},\qquad \Gamma_{\mathrm{phase}}=1.
 \label{eq:memoryphase}
\end{equation}
Projected Adam with the same $k$-bit parameters and $b$-bit first and second moments (\cref{sec:setup-controlled}) stores
\begin{equation}
 M_{\mathrm{proj}}=\frac{(k+2b)N}{8}\ \text{bytes},\qquad
 \Gamma_{\mathrm{proj}}=\frac{M_{\mathrm{proj}}}{M_{\mathrm{phase}}}=1+\frac{2b}{k}.
 \label{eq:memorycontrolled}
\end{equation}
At $k=b=6$ the overhead is exactly 3. Moments of 8 bits and float32 give $11/3\approx3.67$ and $35/3\approx11.67$.

The Adam baselines of the exploratory evaluation differ in one further respect. Their class vectors have free real and imaginary parts, which requires $2CD$ real parameters instead of $CD$ phase indices. With $N_E=D(F+B)$ encoder coordinates and $N_C=DC$ class coordinates, an Adam arm that stores parameters and both moments at $b$ bits holds
\begin{equation}
 M_{\mathrm{Adam},b}=\frac{3b\,(N_E+2N_C)}{8}\ \text{bytes},\qquad
 \frac{M_{\mathrm{Adam},b}}{M_{\mathrm{phase}}}
 =\frac{3b}{k}\cdot\frac{N_E+2N_C}{N_E+N_C}.
 \label{eq:memoryarchive}
\end{equation}
The first factor is the cost of storing moments; the second is the cost of the unconstrained class head, which grows with the share of class coordinates in the model. These formulas exclude gradients, activations, temporary dequantized arrays, scalar quantizer ranges, lookup tables, and allocator overhead. They describe what a packed implementation must retain, not what the floating-point simulation in this paper allocates.

%% file: experiments.tex
\section{Experimental Design}
\label{sec:setup}
The evidence in this paper comes from five lines of study that answer different questions and carry different weight. \Cref{tab:evidence} lists them. The controlled comparison is the only study that isolates the update rule and its stored history under validation-only model selection, and it therefore carries the central claim. The screening studies on digits, the six-dataset evaluation, and the character-level text study map the accuracy and memory trade-off against a wider range of competitors and tasks, including Adam with continuous parameters, at the price of weaker controls. The instrumented replays measure how each method moves the phase memory, which turns open questions about mechanisms into measurements. We keep these levels separate in every table and figure.

\begin{table}[t]
\centering\small
\caption{\textbf{Evidence map.} The five lines of study, the question each answers, and the controls under which it was run. ``Selection'' states which data chose the hyperparameters.}
\label{tab:evidence}
\begin{tabularx}{\textwidth}{p{2.6cm}p{2.7cm}Xp{2.4cm}}
\toprule
Study & Data and scale & Question & Selection; seeds\\
\midrule
Controlled comparison (\cref{sec:setup-controlled}) & digits; $D=64$, $k=6$ & Does removing all optimizer history cost accuracy when parameters, head, and initialization are identical? & separate validation set; 5\\
Screening studies (\cref{sec:setup-screening}) & digits; $D=96$, 20 epochs & Where does \PHDC{} sit among float32, 8-bit, 6-bit, and hybrid training systems? Do generated item memories and fixed-point gradients work? & tuning seed on evaluation split; 5\\
Six-dataset evaluation (\cref{sec:setup-exploratory}) & six public datasets; $D\leq384$ & How large are the accuracy cost and the state saving across tasks, and how does the phase width matter? & tuning seed on evaluation split; 3\\
Character-level text (\cref{sec:setup-text}) & five text corpora; $D=128$, 64-symbol context & Does the trade-off carry over to next-symbol prediction with vocabularies of up to 256 symbols? & tuning seed on test windows; 3\\
Instrumented replays (\cref{sec:setup-diagnostics}) & digits; controlled runs & How often and how far does each method move each coordinate, and why? & replays of the final runs; 5\\
\bottomrule
\end{tabularx}
\end{table}

\subsection{Controlled Equal-Precision Comparison}
\label{sec:setup-controlled}
The controlled experiment uses the 1{,}797-example $8\times8$ handwritten digit dataset distributed with scikit-learn~\citep{pedregosa2011}. Pixel values are integers in $\{0,\ldots,16\}$, so $F=64$ and $B=17$. A stratified split assigns 1{,}077 examples to training, 360 to validation, and 360 to testing. Every method uses $D=64$, $k=6$, $C=10$, $T=0.3$, minibatches of 128 examples, and 20 epochs, which gives 180 updates per run. There is no augmentation, learning-rate schedule, or early stopping, and the validation set is never merged into training.

The competitor is Adam constrained to the same phase grid, which we call \emph{projected Adam}. At step $t$, with minibatch gradient $g_t$, it updates the first moment $m_t$ and second moment $v_t$ of every coordinate and rounds the candidate step onto the grid:
\begin{equation}
\begin{gathered}
 m_t=\beta_1 m_{t-1}+(1-\beta_1)g_t,\qquad
 v_t=\beta_2 v_{t-1}+(1-\beta_2)g_t^{2},\\
 r_t=\frac{m_t/(1-\beta_1^{\,t})}{\sqrt{v_t/(1-\beta_2^{\,t})}+\eps_{\mathrm{A}}},\qquad
 q_t=\rnd\bigl(q_{t-1}-\eta\,r_t\bigr)\bmod K,
\end{gathered}
 \label{eq:projadam}
\end{equation}
with $(\beta_1,\beta_2)=(0.9,0.999)$, $\eps_{\mathrm{A}}=10^{-8}$, and learning rate $\eta$ (separately $\eta_{\mathrm{enc}}$ and $\eta_{\mathrm{cls}}$ for encoder and class tables). Here $r_t$ is the bias-corrected moment ratio and $\rnd$ rounds to the nearest integer. Projected Adam keeps no continuous residual of the parameters, so its parameters are exactly as constrained as those of \PHDC{}. Its moment arrays are stored in float32 or simulated with 8-bit or 6-bit per-tensor affine quantization (\cref{app:quantizer}). The three Adam arms thus differ only in moment precision, and all four methods share the stored parameter alphabet, the class head, the initial integer phase tables of each seed, and the minibatch order.

Each method receives a grid of nine configurations, selected by final-epoch validation accuracy on seed 0. For \PHDC{}, $\tau_{\mathrm{enc}}$ and $\tau_{\mathrm{cls}}$ each range over $\{0,3\times10^{-5},3\times10^{-4}\}$. For Adam, $\eta_{\mathrm{enc}}$ and $\eta_{\mathrm{cls}}$ each range over $\{0.3,1,3\}$. Ties select the first configuration in declared order. All selection ends before any test evaluation, after which the selected configuration runs on seeds 1 to 5. All three Adam arms selected $\eta_{\mathrm{enc}}=\eta_{\mathrm{cls}}=3$, the upper edge of the grid; \cref{sec:discussion-gating} explains why projection onto a grid favors large learning rates. \PHDC{} selected $\tau_{\mathrm{enc}}=3\times10^{-5}$ and $\tau_{\mathrm{cls}}=3\times10^{-4}$. A float64 central-difference check confirms the analytic gradient on all 45 coordinates of a small instance, with maximum absolute discrepancy $4.73\times10^{-11}$.

\subsection{Screening Studies on Digits}
\label{sec:setup-screening}
Three screening studies, also run on the scikit-learn digits, preceded the controlled comparison. They use $D=96$ and 20 epochs. Each method searched a grid of eight configurations on seed 0, and the selected configuration was then evaluated on seeds 1 to 5. The screening records do not document a validation split separate from the evaluation split, so we treat their selection like that of the six-dataset evaluation, as exploratory. Their executable code is not part of the released package; their accuracies come from the archived reports and summary tables, and we recompute every state figure with the accounting of \cref{eq:memoryarchive} so that all ratios in the paper are comparable (\cref{app:screening-accounting}).

\paragraph{Training systems on the accuracy and memory frontier.} The first study compares eight training systems. \emph{Adam-fp32} trains continuous float32 encoder phases and free real and imaginary class components with Adam, the standard way to train such a model. \emph{Affine Adam-8} and \emph{Affine Adam-6} store all parameters and both moments with per-tensor affine quantization at 8 and 6 bits (\cref{app:quantizer}). A \emph{hybrid} system trains a $k$-bit integer encoder with the update of \cref{eq:update} but keeps a free float32 class head trained by Adam, for $k\in\{4,6\}$. Finally, \PHDC{} is run at $k\in\{4,6,8\}$. At $k=8$, a first threshold grid inherited from $k=6$ was replaced by a grid extended down to $\tau=10^{-5}$ after the gradient-scale argument of \cref{eq:precision} was recognized; we report both outcomes.

\paragraph{Generated item memories.} The second study tests whether the item memory can be generated rather than stored. In the manner of fractional binding~\citep{komer2019}, each attribute receives one learned frequency vector, and the phase of symbol $s$ at coordinate $d$ is $s$ times that frequency modulo $K$. Two variants are compared with a dense table at $D=128$: a frequency over the flattened pixel index, and separate row and column frequencies bound together.

\paragraph{Fixed-point gradient arithmetic.} The third study asks whether the gradient itself can be computed in low precision, as an integer device would require. A simulation rounds two groups of intermediate quantities to a fixed number of bits: the cosine and sine lookup table that decodes phases (\emph{table bits}), and the logits, softmax, and probability computations of the head (\emph{softmax bits}). It measures, first, the share of gradient signs that agree with the exact float gradient at one frozen parameter state and, second, the test accuracy after training end to end with the rounded pipeline, with $k=6$ and table bits fixed at 6. Rounding is simulated by quantizing and dequantizing in floating point, which reproduces rounding error but not integer overflow or saturation.

\subsection{Six-Dataset Evaluation}
\label{sec:setup-exploratory}
The six-dataset evaluation uses Letter Recognition~\citep{slate1991}, ISOLET~\citep{cole1991}, human activity recognition from smartphones (HAR)~\citep{reyes2013}, MNIST~\citep{lecun1998}, Fashion-MNIST~\citep{xiao2017}, and CIFAR-10~\citep{krizhevsky2009}. \Cref{tab:datasets} lists the settings. Letter and ISOLET use stratified 70/30 example-level splits with random seed 0; the ISOLET split therefore differs from its standard speaker-disjoint split. HAR and the image datasets use their supplied partitions. Tabular features are discretized with training-set minima and maxima, and image intensities with fixed bins (\cref{app:protocol}).

\begin{table}[t]
\centering\small
\caption{Settings of the six-dataset evaluation. $F$: positions; $B$: value bins; $C$: classes; $D$: phase dimension. The evaluation examples also served for hyperparameter selection.}
\label{tab:datasets}
\begin{tabular}{lrrrrrrrr}
\toprule
Dataset & Train & Eval. & $F$ & $B$ & $C$ & $D$ & Epochs & Batch\\
\midrule
Letter & 14{,}000 & 6{,}000 & 16 & 16 & 26 & 128 & 25 & 256\\
ISOLET & 5{,}457 & 2{,}340 & 617 & 32 & 26 & 256 & 25 & 256\\
HAR & 7{,}352 & 2{,}947 & 561 & 32 & 6 & 256 & 25 & 256\\
MNIST & 60{,}000 & 10{,}000 & 784 & 16 & 10 & 384 & 10 & 128\\
Fashion-MNIST & 60{,}000 & 10{,}000 & 784 & 16 & 10 & 384 & 10 & 128\\
CIFAR-10 & 50{,}000 & 10{,}000 & 1{,}024 & 64 & 10 & 384 & 10 & 64\\
\bottomrule
\end{tabular}
\end{table}

Four training systems are compared, with the same names as in the screening studies. \emph{Phase-HDC-6} is \PHDC{} with $k=6$. \emph{Adam-fp32} learns continuous float32 encoder phases and free complex class components, differentiating the continuous forward map directly. \emph{Affine Adam-8} and \emph{Affine Adam-6} store all parameters and both moments with per-tensor affine quantization, using exponential range tracking with rate $\rho=0.01$, and dequantize before each floating-point computation. Their phase parameters need not lie on the equally spaced grid of \PHDC{}, and they do not use block-wise quantization~\citep{dettmers2022}. A follow-up sweep runs \PHDC{} at $k\in\{2,4,8\}$ with the same threshold grid as $k=6$.

Each Adam arm searched eight learning-rate pairs and \PHDC{} nine threshold pairs on seed 0, and the selected setting then ran on three further seeds. The main comparison of Adam-fp32, Affine Adam-8, and Phase-HDC-6 completed in a single 7.9-hour session on one T4-class GPU; Affine Adam-6 and the width sweep ran in two later sessions under the same protocol. Selection used the same evaluation examples that are later reported. A change of seed does not make those examples held out, and selection bias of this kind can be as large as the differences between algorithms~\citep{cawley2010}. The class heads also differ between \PHDC{} and the Adam arms. We therefore treat this evaluation as a comparison of complete training systems that indicates trends, not as an isolated test of moment removal.

\subsection{Character-Level Text Prediction}
\label{sec:setup-text}
The text study asks whether the trade-off holds on a task family that differs from image and tabular classification in two ways: the inputs are sequences of discrete symbols, and the number of classes equals the size of the vocabulary. The task is next-symbol prediction. Each input is a window of the 64 preceding symbols, so $F=64$ positions, each holding a symbol from a vocabulary of size $B$; the label is the following symbol, so $C=B$. The phase memory therefore learns one vector per context position, one per symbol as an input value, and one per symbol as a class. Because bundling sums the bound phasors of all positions, the model sees which symbol occurs at which offset but has no recurrence.

Five corpora are used, each truncated to its first $10^6$ characters, or bytes for enwik8: Tiny Shakespeare~\citep{karpathy2015}, TinyStories~\citep{eldan2023}, the language-modeling version of the Penn Treebank (PTB)~\citep{marcus1993}, and text8 and enwik8 from the Large Text Compression Benchmark~\citep{mahoney2011}. Tiny Shakespeare, TinyStories, and PTB are restricted to printable ASCII characters plus newline and tab, with any other character replaced by a space, which gives vocabularies of 65, 70, and 49 symbols. text8 contains only lowercase letters and spaces (27 symbols). enwik8 is raw Wikipedia markup modeled at the byte level, so its vocabulary has all 256 byte values, and many of them, such as the bytes that encode non-ASCII characters, are rare. The last $10\%$ of each corpus is held out as a contiguous test segment, so that no test window overlaps the training text. From each part, 150{,}000 training windows and 15{,}000 test windows are sampled.

All systems use $D=128$, $k=6$, $T=0.3$, four epochs, and minibatches of 256 windows. The three training systems are Adam-fp32, Affine Adam-8, and Phase-HDC-6, defined as in \cref{sec:setup-exploratory}. Each searched a grid of four configurations on seed 0, evaluated on the test windows, and the selected configuration then ran on three further seeds. Selection therefore reuses the reported test windows, as in the six-dataset evaluation. The complete study ran in about 21 minutes on one T4-class GPU. A sixth corpus, WikiText-2, was dropped because its loader failed on the run's library version; no result for it exists.

\subsection{Update-Dynamics Instrumentation}
\label{sec:setup-diagnostics}
To measure mechanisms rather than infer them, we replay the 20 final controlled runs (four methods, five seeds) with instrumentation. The replay uses the released split, selected configurations, and seeds, touches only training data, and reproduces every released final checkpoint bit for bit. For each step and table it records the share of coordinates whose index changes (the \emph{write fraction}), the distribution of circular displacements in notches, and the training loss after each epoch. For the quantized-moment arms it also records the share of second-moment entries stored at the lowest quantization code and the share of first-moment entries whose sign changes under quantization. A second replay of the 6-bit arm resolves these quantities by row of the value table, that is, by pixel value. A third check reruns projected Adam at the smallest learning rates of the controlled grid on the tuning seed and counts how many stored indices ever leave their initial values.

\subsection{Statistical Reporting}
\label{sec:setup-stats}
All reported numbers are computed from individual seed records where these exist, and from the archived summary tables for the screening studies. We report means and sample standard deviations (divisor $n-1$) over seeds; on a fixed data split these describe variation over initialization and minibatch order, not generalization uncertainty. For the controlled comparison we add a paired Student-$t$ interval over the five seeds. For the six-dataset evaluation we add a two-sided Wilcoxon signed-rank test over datasets; with six datasets its smallest attainable $p$-value is $1/32\approx0.031$, so it can establish at most that a direction of difference never reverses. The same test over all eleven image, tabular, and text datasets has a smallest attainable $p$-value of $1/1024$. State is reported in KiB ($1{,}024$ bytes) with the accounting of \cref{sec:accounting}. Averages over datasets weight each dataset equally, and an average over a subset of datasets is always labeled as such.

\section{Results}
\label{sec:results}
The results are ordered from the most to the least controlled comparison. \Cref{sec:results-controlled} establishes what removing all optimizer history costs when everything else is held fixed. \Cref{sec:results-screening,sec:results-fp32,sec:results-affine,sec:results-text} place \PHDC{} among realistic training systems on digits, on six image and tabular datasets, and on five text corpora. \Cref{sec:results-width,sec:results-generated,sec:results-fixedpoint} examine three design choices: the phase width, the stored item memory, and the precision of gradient arithmetic.

\subsection{Parity-Level Accuracy at One Third of the Persistent State}
\label{sec:results-controlled}
\Cref{tab:controlled} and \cref{fig:controlled} report the controlled comparison. \PHDC{} reaches $92.44\pm0.93\%$ test accuracy and projected Adam with 6-bit moments reaches $91.22\pm1.19\%$. Both store the same 5{,}824 six-bit phase indices; projected Adam adds two arrays of equal size, so the persistent state is 4.27~KiB against 12.80~KiB. The factor of three is an exact consequence of \cref{eq:memorycontrolled}.

\begin{table}[t]
\centering\small
\caption{\textbf{Controlled digit comparison.} All methods share the 6-bit phase grid, class head, initial tables, and minibatch order; only the update and its stored moments differ. Accuracy is mean $\pm$ sample SD over five seeds after validation-only selection. State counts packed parameter and moment arrays. Ratio is state relative to \PHDC{}. Bold marks the best value in each column.}
\label{tab:controlled}
\begin{tabular}{lrrr}
\toprule
Method & Test accuracy (\%) & State (KiB) & Ratio\\
\midrule
\input{tables/controlled_digits}
\end{tabular}
\end{table}

The paired mean difference over the 6-bit-moment arm is $1.22$ percentage points, with a 95\% interval of $[-0.44,2.89]$. The interval contains zero, and five seeds on one split cannot establish either superiority or equivalence. What the data do establish is that removing both moment arrays does not reduce the mean accuracy of this classifier at this budget. The 8-bit and float32 moment arms reach $91.44\%$ and $90.78\%$. Their order is not monotone in precision, which already suggests that moment precision is not what limits projected Adam here; \cref{sec:discussion-gating} identifies what does.

The validation grid shows a second difference. Across its nine configurations, \PHDC{} reached between $85.3\%$ and $94.7\%$ validation accuracy. Projected Adam ranged from $9.7\%$ to $92.2\%$ with float32 moments, from $48.3\%$ to $94.7\%$ with 6-bit moments, and from $43.6\%$ to $93.6\%$ with 8-bit moments. The lowest values belong to the smallest learning rates, and \cref{sec:discussion-gating} shows that at these rates projection prevents the parameters from moving at all.

\begin{figure}[t]
\centering
\includegraphics[width=\textwidth]{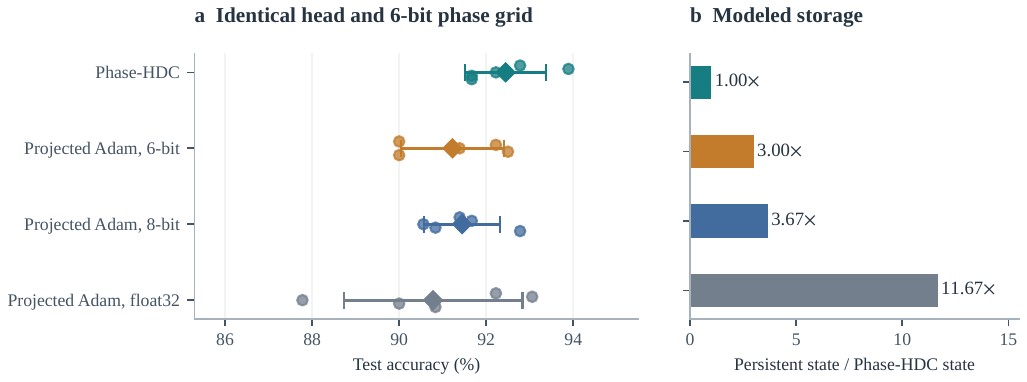}
\caption{\textbf{Controlled comparison on a common 6-bit phase grid.} (a)~Test accuracy of the five final seeds per method; diamonds and bars give mean $\pm$ sample SD. (b)~Persistent parameter and moment state relative to \PHDC{}. All methods share parameters, head, and initialization, and were tuned on a separate validation set. Ratios are logical storage, not measured speed or memory.}
\label{fig:controlled}
\end{figure}

\subsection{Accuracy and Memory Frontier on Digits}
\label{sec:results-screening}
The screening study places \PHDC{} among eight training systems (\cref{tab:screening} and \cref{fig:screening}a). Adam-fp32 is the most accurate at $96.48\pm0.56\%$, and it also keeps the most state, 113.6~KiB. \PHDC{} at $k=6$ reaches $93.70\pm0.94\%$ with 6.4~KiB, which is $2.78$ points lower at $17.8\times$ less state. At $k=4$, \PHDC{} keeps only 4.3~KiB, $26.6\times$ less than Adam-fp32, but its accuracy falls to $84.81\pm2.72\%$. At $k=8$, it reaches $93.81\pm1.07\%$ with 8.5~KiB, a $13.3\times$ saving.

\begin{table}[t]
\centering\small
\caption{\textbf{Screening frontier on digits} ($D=96$, five seeds). Gap is the accuracy difference to Adam-fp32 in percentage points; ratio is Adam-fp32 state divided by each system's state. Hybrid systems train an integer encoder with \cref{eq:update} and a float32 class head with Adam. State follows \cref{eq:memoryarchive}. Bold marks the best accuracy and the smallest state.}
\label{tab:screening}
\begin{tabular}{lrrrr}
\toprule
System & Accuracy (\%) & Gap (pp) & State (KiB) & Ratio\\
\midrule
\input{tables/screening_frontier}
\end{tabular}
\end{table}

Against quantized Adam the picture changes with precision. At equal 6-bit precision, \PHDC{} is $0.55$ points more accurate than Affine Adam-6 ($93.70\%$ against $93.15\%$), keeps $3.3\times$ less state, and varies three times less across seeds (SD $0.94$ against $2.83$). Against Affine Adam-8, it is $0.71$ points less accurate and keeps $4.4\times$ less state; at $k=8$ the deficit is $0.60$ points at $3.3\times$ less state. The hybrid systems sit between these points: an integer encoder with a float32 class head reaches $94.33\%$ at $k=6$ and $94.89\%$ at $k=4$, but keeps more than 26~KiB because its class head carries float32 parameters and moments. \Cref{sec:discussion-gap} uses these intermediate systems to locate where the accuracy gap to Adam-fp32 arises.

The $k=8$ result depends on the threshold grid. With the grid inherited from $k=6$, \PHDC{} reached only $86.22\%$ at $k=8$. An extended grid reaching down to $\tau=10^{-5}$ raised it to $93.81\%$. This is the effect predicted by \cref{eq:precision}: two more bits shrink the index-space gradient by a factor of four, so a threshold tuned for $k=6$ is four times too strict at $k=8$.

\begin{figure}[t]
\centering
\includegraphics[width=\textwidth]{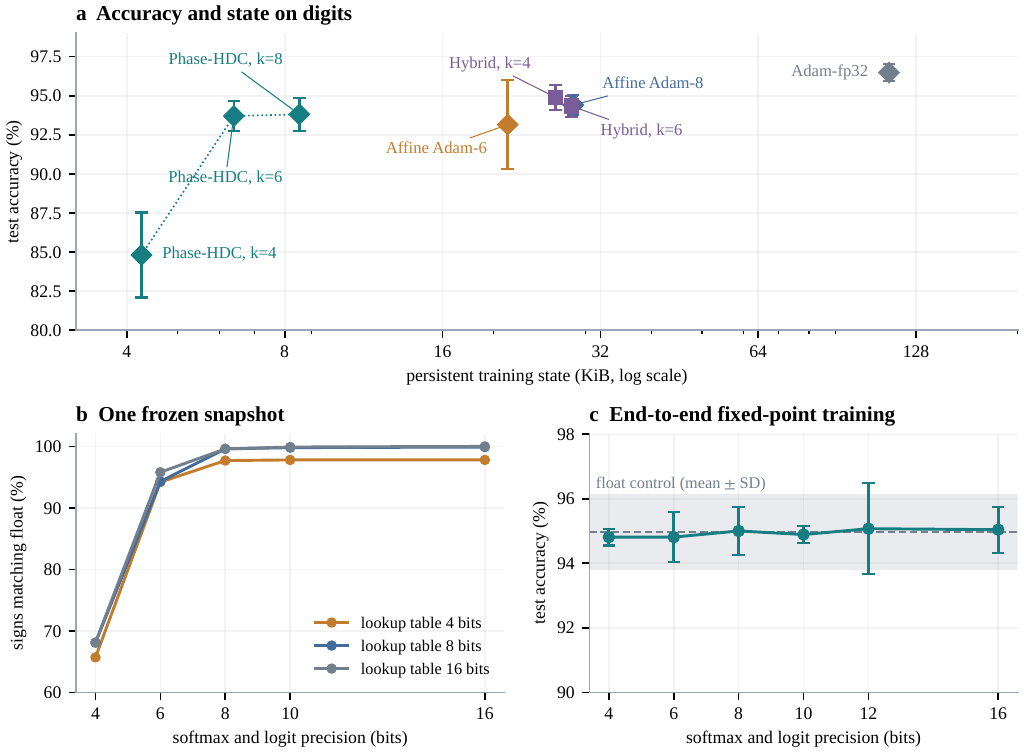}
\caption{\textbf{Screening studies on digits.} (a)~Accuracy against persistent training state for eight training systems (mean $\pm$ SD over five seeds; logarithmic horizontal axis). The dotted line connects \PHDC{} at $k=4$, 6, and 8. (b)~Share of gradient signs that agree with the exact float gradient at one frozen parameter state, as the softmax and lookup-table precisions vary. (c)~Test accuracy after end-to-end training with the rounded pipeline (table bits fixed at 6); the shaded band is the float control mean $\pm$ SD under the same threshold grid.}
\label{fig:screening}
\end{figure}

\subsection{Accuracy Cost Relative to Float32 and 8-bit Adam}
\label{sec:results-fp32}
The six-dataset evaluation repeats this comparison on harder tasks (\cref{tab:gpuaccuracy,tab:gpustate}). Adam-fp32 is more accurate than \PHDC{} on every dataset, by $7.00$ points on average. The gap ranges from $2.61$ points on Fashion-MNIST to $15.64$ points on ISOLET, and its direction never reverses (Wilcoxon $p=0.031$, the smallest value attainable with six datasets). In exchange, \PHDC{} keeps $16.2$ to $23.2\times$ less persistent state. The largest ratio occurs on Letter, whose class table is large relative to its encoder, so that the doubled class representation of Adam weighs more (\cref{eq:memoryarchive}). These are material costs, and the ISOLET gap in particular rules out a general claim that optimizer history can be removed at negligible accuracy cost.

\begin{table}[t]
\centering\small
\caption{\textbf{Accuracy (\%) on six datasets.} Mean $\pm$ sample SD over three seeds. Bold marks the best system per dataset; underline marks the better of the two 6-bit systems when it is not the overall best. Evaluation data were reused for tuning, and the Adam arms have free complex class vectors, so entries compare complete systems.}
\label{tab:gpuaccuracy}
\begin{tabular}{lrrrr}
\toprule
Dataset & Adam-fp32 & Affine Adam-8 & Affine Adam-6 & Phase-HDC-6\\
\midrule
\input{tables/gpu_accuracy}
\end{tabular}
\end{table}

\begin{table}[t]
\centering\small
\caption{\textbf{Persistent state on six datasets (KiB)} and its ratio to Phase-HDC-6 for each Adam system. Ideal packed parameter and moment arrays, not measured allocations. Bold marks the smallest state per dataset.}
\label{tab:gpustate}
\begin{tabular}{lrrrrrrr}
\toprule
 & \multicolumn{4}{c}{State (KiB)} & \multicolumn{3}{c}{Ratio to Phase-HDC-6}\\
\cmidrule(lr){2-5}\cmidrule(lr){6-8}
Dataset & Adam-fp32 & Adam-8 & Adam-6 & Phase-HDC-6 & fp32 & 8-bit & 6-bit\\
\midrule
\input{tables/gpu_state}
\end{tabular}
\end{table}

Against Affine Adam-8, the mean deficit of \PHDC{} narrows to $2.89$ points while its state is $4.0$ to $5.8\times$ smaller. \PHDC{} is ahead on MNIST by $1.50$ points and on CIFAR-10 by $9.54$ points, and behind on the other four datasets; the direction is not consistent (Wilcoxon $p=0.44$). CIFAR-10 shows the largest favorable difference, $32.54\%$ against $23.00\%$, yet Adam-fp32 reaches $37.04\%$ on the same task. \Cref{fig:gpu} places every system in the accuracy and state plane so that a storage ratio is always read together with the accuracy it costs.

\begin{figure}[t]
\centering
\includegraphics[width=\textwidth]{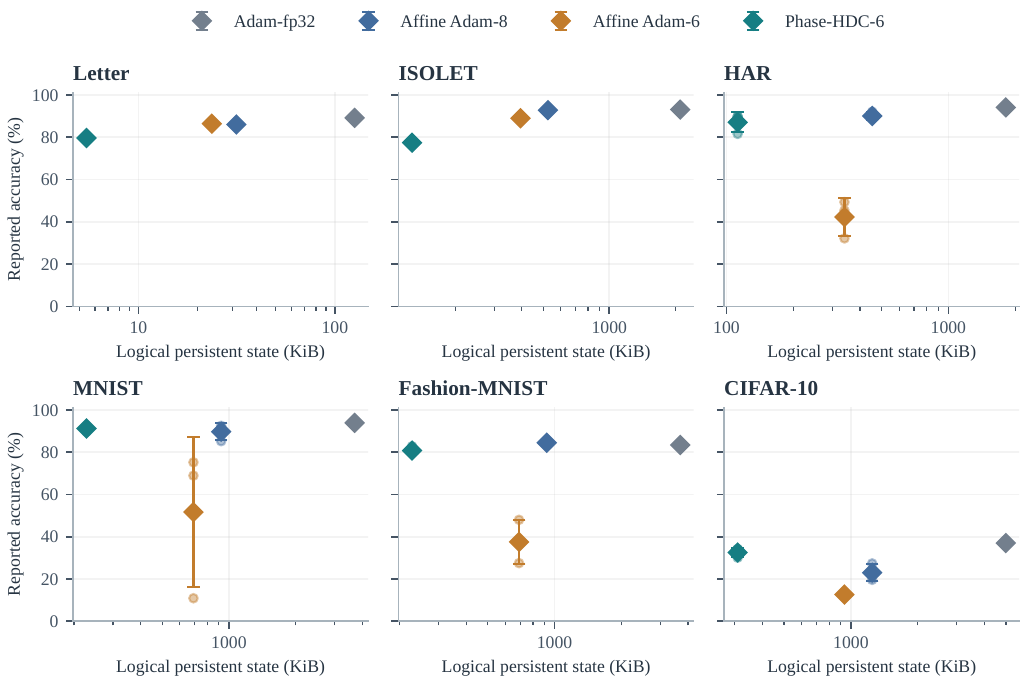}
\caption{\textbf{Accuracy against persistent state on six datasets.} Diamonds and bars show mean $\pm$ sample SD; faint points are individual seeds. The horizontal axis is logical persistent state on a logarithmic scale. Hyperparameter selection reused evaluation data, and class heads differ between \PHDC{} and the Adam arms. Adam-fp32 is the most accurate system on five datasets and Affine Adam-8 on Fashion-MNIST; Affine Adam-6 shows large, dataset-dependent losses.}
\label{fig:gpu}
\end{figure}

\subsection{Collapse of Per-Tensor 6-bit Adam}
\label{sec:results-affine}
At equal 6-bit precision, the comparison on six datasets is far more lopsided than on digits. \PHDC{} has higher mean accuracy than Affine Adam-6 on HAR, MNIST, Fashion-MNIST, and CIFAR-10, by $44.80$, $39.52$, $43.24$, and $19.87$ points, an average of $36.86$ points on these four datasets. On Letter and ISOLET, \PHDC{} trails by $6.77$ and $11.55$ points. Averaged over all six datasets the difference is $21.52$ points in favor of \PHDC{} (Wilcoxon $p=0.16$, because the direction reverses on two datasets). Throughout, \PHDC{} keeps $3.0$ to $4.3\times$ less state.

The losses of Affine Adam-6 are not graceful. On MNIST its three seeds reach $75.24\%$, $68.94\%$, and $10.90\%$, a sample SD of $35.47$ points against $1.18$ for \PHDC{}. Variability is not negligible for \PHDC{} either: its SD on HAR is $4.74$ points. The tuning records show that the collapse is not caused by one unlucky configuration. Across its eight configurations on the tuning seed, Affine Adam-6 never exceeded $72.0\%$ on HAR, $72.9\%$ on MNIST, $60.7\%$ on Fashion-MNIST, or $13.9\%$ on CIFAR-10, whereas the best \PHDC{} configurations reached $90.0\%$, $91.7\%$, $83.8\%$, and $31.7\%$ (\cref{tab:sensitivity}). The instrumented replays in \cref{sec:discussion-quantization} identify a mechanism that explains why this failure appears on some datasets and not on others.

\subsection{Character-Level Text Prediction}
\label{sec:results-text}
The text study reproduces the trade-off on a different task family (\cref{tab:text,tab:textstate} and \cref{fig:text}). All three systems are far above the uniform-guessing level of one over the vocabulary size on every corpus. That level is a lenient reference, because the most frequent symbol, typically the space, alone accounts for a sizable share of English text, so the informative comparisons are those between training systems.

Adam-fp32 is the most accurate system on all five corpora. \PHDC{} trails it by $1.65$ points on text8 to $4.97$ points on enwik8, $3.28$ points on average, a smaller gap than on the image and tabular datasets, while keeping $19.7$ to $23.1\times$ less state. Against Affine Adam-8, \PHDC{} is more accurate on four of five corpora, by $8.9$ to $10.5$ points on Tiny Shakespeare, TinyStories, and PTB and by $19.5$ points on enwik8, and it trails by $0.60$ points on text8. It keeps $4.9$ to $5.8\times$ less state, and the ratio grows with the vocabulary because the free class head of Adam stores two real coordinates and two moments per class coordinate.

\begin{table}[t]
\centering\small
\caption{\textbf{Next-symbol accuracy (\%) on five text corpora.} Mean $\pm$ sample SD over three seeds. $B$ is the vocabulary size and $100/B$ the uniform-guessing level. Bold marks the best system per corpus. Selection reused the test windows.}
\label{tab:text}
\begin{tabular}{lrrrrr}
\toprule
Corpus & $B$ & $100/B$ & Adam-fp32 & Affine Adam-8 & Phase-HDC-6\\
\midrule
\input{tables/text_accuracy}
\end{tabular}
\end{table}

\begin{table}[t]
\centering\small
\caption{\textbf{Persistent state on five text corpora (KiB)} and its ratio to Phase-HDC-6. Bold marks the smallest state per corpus.}
\label{tab:textstate}
\begin{tabular}{lrrrrr}
\toprule
 & \multicolumn{3}{c}{State (KiB)} & \multicolumn{2}{c}{Ratio to Phase-HDC-6}\\
\cmidrule(lr){2-4}\cmidrule(lr){5-6}
Corpus & Adam-fp32 & Adam-8 & Phase-HDC-6 & fp32 & 8-bit\\
\midrule
\input{tables/text_state}
\end{tabular}
\end{table}

enwik8 stands out. With 256 byte values, it has the largest vocabulary and therefore the largest class table of any task in this paper. Affine Adam-8 collapses there: its three seeds reach $1.44\%$, $0.01\%$, and $0.90\%$, barely above the $0.39\%$ guessing level, and even its best tuning configuration reached only $12.6\%$ on the tuning seed. \PHDC{} reaches $20.26\pm0.44\%$ and Adam-fp32 $25.23\pm0.34\%$. Unlike the image collapses of \cref{sec:results-affine}, which occurred at 6 bits, this one occurs with 8-bit moments. text8 is the opposite case: with only 27 symbols, all of them common, Affine Adam-8 stays within $1.05$ points of Adam-fp32. \Cref{sec:discussion-quantization} relates both cases to the frequency mechanism measured on digits.

The six image and tabular datasets and the five corpora together give eleven datasets. Adam-fp32 is more accurate than \PHDC{} on all eleven (Wilcoxon $p=0.001$), by $5.31$ points on average. \PHDC{} is more accurate than Affine Adam-8 on six of the eleven, by $2.85$ points on average, but the direction is not consistent (Wilcoxon $p=0.41$).

\begin{figure}[t]
\centering
\includegraphics[width=\textwidth]{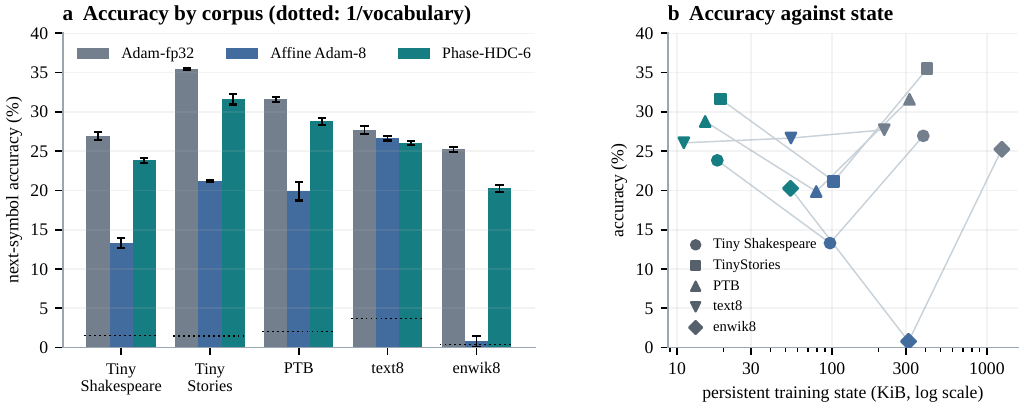}
\caption{\textbf{Character-level text prediction on five corpora.} (a)~Next-symbol accuracy (mean $\pm$ SD over three seeds); dotted lines mark one over the vocabulary size. (b)~Accuracy against persistent training state; lines connect the three systems on the same corpus, colors follow (a), and markers identify the corpus. Affine Adam-8 collapses on byte-level enwik8.}
\label{fig:text}
\end{figure}

\subsection{Phase Width on Six Datasets}
\label{sec:results-width}
\Cref{fig:width} varies the phase width $k\in\{2,4,6,8\}$ for \PHDC{} on the six datasets. At $k=2$, accuracy stays near the class-frequency level everywhere. At $k=4$, Letter reaches $69.63\%$, while MNIST and Fashion-MNIST reach only $16.06\%$ and $17.90\%$ with SDs of $11.65$ and $9.61$ points. At $k=6$, all six tasks improve substantially. At $k=8$, accuracy improves slightly on the three image datasets and falls on the three tabular datasets, most sharply on ISOLET, from $77.49\%$ to $57.72\%$.

These results do not show that four phase levels are too few to represent a useful classifier. By \cref{eq:precision}, the sweep changes the representation, the angular step ($90^\circ$, $22.5^\circ$, $5.625^\circ$, and $1.40625^\circ$), and the angular gate together, because the same threshold grid is reused at every width. At $k=8$, the selected encoder threshold lies at the smallest grid value on five of six datasets, the same signature that the extended grid resolved on digits (\cref{sec:results-screening}). \Cref{sec:discussion-width} analyzes which effects account for the failures at two, four, and eight bits.

\begin{figure}[t]
\centering
\includegraphics[width=\textwidth]{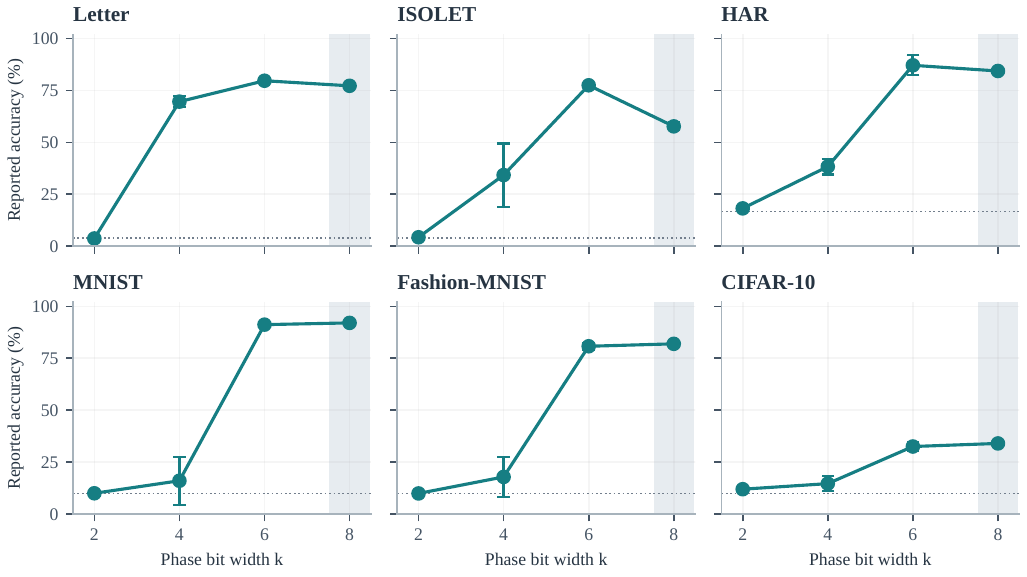}
\caption{\textbf{Phase-width sweep of \PHDC{} on six datasets.} Mean accuracy and sample SD over three seeds. Dotted lines mark $1/C$. The shaded column marks $k=8$, where the selected encoder threshold lies at the grid boundary on five of six datasets. Step size and effective gate change with $k$ (\cref{eq:precision}), so the sweep is not an isolated capacity experiment. Exact values are in \cref{tab:width}.}
\label{fig:width}
\end{figure}

\subsection{Generated Item Memories}
\label{sec:results-generated}
A generated item memory would store one frequency vector per attribute instead of one vector per symbol, which would shrink the item memory by orders of magnitude. \Cref{tab:generated} shows that this saving destroys accuracy on digits. The flattened-index variant reaches $38\%$ to $55\%$ across bit widths, and the row-and-column variant reaches $72\%$ at $k=8$ even with twice the dimension. The dense table reaches $94\%$ to $96\%$ at every width, 25 to 45 points higher. The reason is structural. A single frequency places the codes of all symbols of an attribute on one helix in phase space, so that neighboring symbols receive similar codes. That property makes fractional binding useful for continuous space~\citep{komer2019}, but a classifier needs the opposite: each pixel value and position must be free to receive whatever code best separates the classes. All other results in this paper therefore use dense tables.

\begin{table}[t]
\centering\small
\caption{\textbf{Generated against dense item memories on digits} (test accuracy, \%, $D=128$). A generated memory derives every symbol's phase from one learned frequency per attribute. $^{\dagger}$At $D=256$. Bold marks the best memory per width.}
\label{tab:generated}
\begin{tabular}{lrrrr}
\toprule
Item memory & $k=2$ & $k=4$ & $k=6$ & $k=8$\\
\midrule
\input{tables/screening_generated}
\end{tabular}
\end{table}

\subsection{Fixed-Point Gradient Arithmetic}
\label{sec:results-fixedpoint}
An integer device must compute the gradient itself in low precision, not only store the parameters that way. At one frozen parameter state, low precision visibly corrupts the gradient (\cref{fig:screening}b and \cref{tab:fpsnapshot}). With 4-bit softmax arithmetic, only $65.7\%$ to $68.1\%$ of gradient signs agree with the exact float gradient, whatever the precision of the lookup table. With 8-bit softmax arithmetic the agreement rises to $97.7\%$ to $99.6\%$. The precision of the head, not of the trigonometric table, sets the agreement.

\begin{table}[t]
\centering\small
\caption{\textbf{Gradient sign agreement (\%) with the exact float gradient} at one frozen parameter state, by lookup-table precision (rows) and softmax precision (columns).}
\label{tab:fpsnapshot}
\begin{tabular}{lrrrrr}
\toprule
Table bits & 4 & 6 & 8 & 10 & 16\\
\midrule
\input{tables/screening_fp_snapshot}
\end{tabular}
\end{table}

End-to-end training tells a different story (\cref{fig:screening}c and \cref{tab:fptraining}). With table bits fixed at 6, every softmax precision from 16 down to 4 bits trains to between $94.81\%$ and $95.07\%$, inside one standard deviation of the float control at $94.96\pm1.17\%$. A configuration with a 4-bit table and 8-bit softmax reached $95.89\pm0.64\%$. Training is thus far more tolerant of rounded gradients than a single-state agreement measurement suggests. The float control here is $1.26$ points higher than the $93.70\%$ of \cref{tab:screening} because this study used a finer threshold grid, which again shows how much \PHDC{} depends on its threshold.

\begin{table}[!htbp]
\centering\small
\caption{\textbf{Test accuracy (\%) after end-to-end training with rounded gradient arithmetic} (table bits 6, five seeds). The control computes in float with the same threshold grid.}
\label{tab:fptraining}
\begin{tabular}{lr}
\toprule
Softmax bits & Accuracy (\%)\\
\midrule
\input{tables/screening_fp_training}
\end{tabular}
\end{table}

%% file: tables/controlled_digits.tex
Phase-HDC & $\boldsymbol{92.44 \pm 0.93}$ & \textbf{4.27} & $\mathbf{1.00}\boldsymbol{\times}$ \\
Projected Adam, 6-bit moments & $91.22 \pm 1.19$ & 12.80 & $3.00\times$ \\
Projected Adam, 8-bit moments & $91.44 \pm 0.87$ & 15.64 & $3.67\times$ \\
Projected Adam, float32 moments & $90.78 \pm 2.05$ & 49.77 & $11.67\times$ \\
\bottomrule

%% file: tables/screening_frontier.tex
Adam-fp32 & $\boldsymbol{96.48 \pm 0.56}$ & --- & 113.62 & 1.0$\times$ \\
Affine Adam-8 & $94.41 \pm 0.65$ & $-2.07$ & 28.41 & 4.0$\times$ \\
Affine Adam-6 & $93.15 \pm 2.83$ & $-3.33$ & 21.30 & 5.3$\times$ \\
\midrule
Hybrid, $k=4$ & $94.89 \pm 0.80$ & $-1.59$ & 26.30 & 4.3$\times$ \\
Hybrid, $k=6$ & $94.33 \pm 0.67$ & $-2.15$ & 28.20 & 4.0$\times$ \\
\midrule
Phase-HDC, $k=4$ & $84.81 \pm 2.72$ & $-11.67$ & \textbf{4.27} & \textbf{26.6}$\boldsymbol{\times}$ \\
Phase-HDC, $k=6$ & $93.70 \pm 0.94$ & $-2.78$ & 6.40 & 17.8$\times$ \\
Phase-HDC, $k=8$ & $93.81 \pm 1.07$ & $-2.67$ & 8.53 & 13.3$\times$ \\
\bottomrule

%% file: tables/gpu_accuracy.tex
Letter & $\boldsymbol{89.19 \pm 0.18}$ & $86.06 \pm 0.44$ & $\underline{86.45 \pm 0.44}$ & $79.68 \pm 0.25$ \\
ISOLET & $\boldsymbol{93.13 \pm 0.31}$ & $92.86 \pm 0.41$ & $\underline{89.05 \pm 0.54}$ & $77.49 \pm 0.75$ \\
HAR & $\boldsymbol{94.16 \pm 0.09}$ & $90.09 \pm 1.73$ & $42.31 \pm 9.01$ & $\underline{87.12 \pm 4.74}$ \\
MNIST & $\boldsymbol{93.91 \pm 0.20}$ & $89.72 \pm 3.92$ & $51.69 \pm 35.47$ & $\underline{91.22 \pm 1.18}$ \\
Fashion-MNIST & $83.42 \pm 0.86$ & $\boldsymbol{84.48 \pm 0.65}$ & $37.56 \pm 10.25$ & $\underline{80.81 \pm 1.73}$ \\
CIFAR-10 & $\boldsymbol{37.04 \pm 1.06}$ & $23.00 \pm 3.96$ & $12.67 \pm 0.86$ & $\underline{32.54 \pm 2.14}$ \\
\bottomrule

%% file: tables/gpu_state.tex
Letter & 126.00 & 31.50 & 23.62 & \textbf{5.44} & $23.17\times$ & $5.79\times$ & $4.34\times$ \\
ISOLET & 2103.00 & 525.75 & 394.31 & \textbf{126.56} & $16.62\times$ & $4.15\times$ & $3.12\times$ \\
HAR & 1815.00 & 453.75 & 340.31 & \textbf{112.31} & $16.16\times$ & $4.04\times$ & $3.03\times$ \\
MNIST & 3690.00 & 922.50 & 691.88 & \textbf{227.81} & $16.20\times$ & $4.05\times$ & $3.04\times$ \\
Fashion-MNIST & 3690.00 & 922.50 & 691.88 & \textbf{227.81} & $16.20\times$ & $4.05\times$ & $3.04\times$ \\
CIFAR-10 & 4986.00 & 1246.50 & 934.88 & \textbf{308.81} & $16.15\times$ & $4.04\times$ & $3.03\times$ \\
\bottomrule

%% file: tables/text_accuracy.tex
Tiny Shakespeare & 65 & 1.5 & $\boldsymbol{26.94 \pm 0.48}$ & $13.29 \pm 0.60$ & $23.82 \pm 0.34$ \\
TinyStories & 70 & 1.4 & $\boldsymbol{35.47 \pm 0.16}$ & $21.17 \pm 0.16$ & $31.61 \pm 0.68$ \\
PTB & 49 & 2.0 & $\boldsymbol{31.59 \pm 0.35}$ & $19.87 \pm 1.15$ & $28.78 \pm 0.43$ \\
text8 & 27 & 3.7 & $\boldsymbol{27.70 \pm 0.53}$ & $26.66 \pm 0.31$ & $26.05 \pm 0.23$ \\
enwik8 & 256 & 0.4 & $\boldsymbol{25.23 \pm 0.34}$ & $0.78 \pm 0.72$ & $20.26 \pm 0.44$ \\
\bottomrule

%% file: tables/text_state.tex
Tiny Shakespeare & 388.5 & 97.1 & \textbf{18.2} & $21.4\times$ & $5.3\times$ \\
TinyStories & 411.0 & 102.8 & \textbf{19.1} & $21.5\times$ & $5.4\times$ \\
PTB & 316.5 & 79.1 & \textbf{15.2} & $20.8\times$ & $5.2\times$ \\
text8 & 217.5 & 54.4 & \textbf{11.1} & $19.7\times$ & $4.9\times$ \\
enwik8 & 1248.0 & 312.0 & \textbf{54.0} & $23.1\times$ & $5.8\times$ \\
\bottomrule

%% file: tables/screening_generated.tex
Generated, flattened 1D index & 38 & 41 & 53 & 55 \\
Generated, row and column frequencies & 50 & --- & --- & 72$^{\dagger}$ \\
Dense table & \textbf{96} & \textbf{96} & \textbf{95} & \textbf{94} \\
\bottomrule

%% file: tables/screening_fp_snapshot.tex
4 & 65.7 & 94.2 & 97.7 & 97.8 & 97.8 \\
8 & 68.1 & 94.3 & 99.6 & 99.8 & 99.9 \\
16 & 68.1 & 95.8 & 99.6 & 99.9 & 100.0 \\
\bottomrule

%% file: tables/screening_fp_training.tex
Float (control) & $94.96 \pm 1.17$ \\
\midrule
16 & $95.04 \pm 0.71$ \\
12 & $95.07 \pm 1.41$ \\
10 & $94.89 \pm 0.25$ \\
8 & $95.00 \pm 0.74$ \\
6 & $94.81 \pm 0.78$ \\
4 & $94.81 \pm 0.26$ \\
\bottomrule

%% file: discussion.tex
\section{Discussion}
\label{sec:discussion}
The results raise four questions that final accuracies cannot answer. Why does an update with no history match one with two moment arrays in the controlled setting? Why does per-tensor 6-bit Adam fail badly on some datasets and not on others? Where does the remaining accuracy gap to float32 Adam come from? Why does the phase width affect accuracy so unevenly across tasks? The instrumented replays of \cref{sec:setup-diagnostics}, the intermediate systems of the screening study, and the update geometry of \cref{sec:method} answer the first three and constrain the fourth. \Cref{fig:dynamics} summarizes the replay measurements. A final subsection asks what the results imply for a physical implementation.

\begin{figure}[t]
\centering
\includegraphics[width=\textwidth]{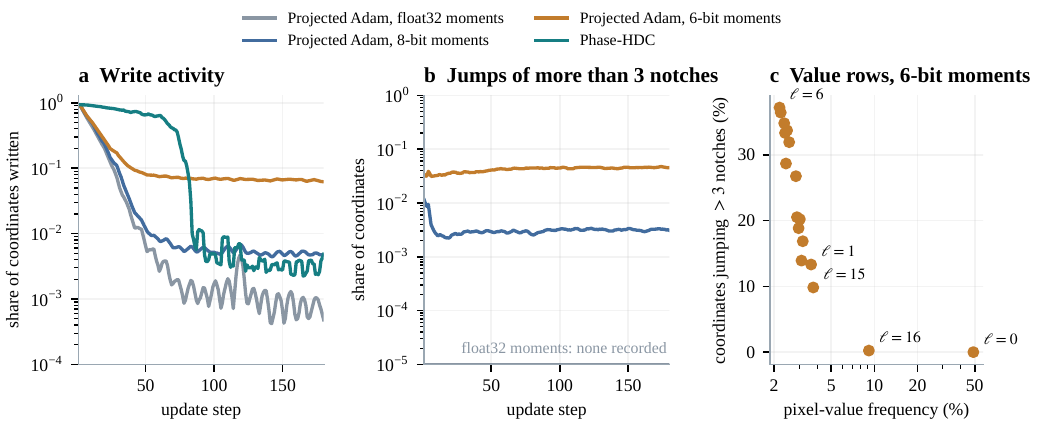}
\caption{\textbf{Update dynamics in the controlled comparison.} Instrumented replays of the 20 final runs, which reproduce the released checkpoints exactly; curves are means over five seeds, smoothed over five steps. (a)~Share of all phase coordinates whose stored index changes at each step. (b)~Share of coordinates displaced by more than three notches in one step; float32 moments produced no such jump. (c)~For projected Adam with 6-bit moments, each point is one row of the value table, that is, one pixel value $\ell$. Rarely observed values are displaced far more often (Spearman rank correlation $-0.98$ over 17 rows).}
\label{fig:dynamics}
\end{figure}

\subsection{Projected Adam as a History-Gated Discrete Update}
\label{sec:discussion-gating}
Projection onto the phase grid changes what Adam's moments do. At the first step, bias correction gives $m_1/(1-\beta_1)=g_1$ and $v_1/(1-\beta_2)=g_1^2$, so the moment ratio of \cref{eq:projadam} is
\begin{equation}
 r_1=\frac{g_1}{|g_1|+\eps_{\mathrm{A}}}=\sign(g_1)\bigl(1-O(\eps_{\mathrm{A}}/|g_1|)\bigr).
 \label{eq:firststep}
\end{equation}
With the selected learning rate $\eta=3$, the first projected step is therefore a three-notch sign step on every coordinate whose gradient exceeds $10^{-8}$. The replays confirm this: at step 1, 99.99\% of coordinates move by exactly three notches in all three Adam arms. Projected Adam begins as unthresholded sign descent with a larger step than \PHDC{}.

After the first step, rounding creates a dead zone. Coordinate $j$ moves at step $t$ only if $\eta\,|r_{t,j}|\geq\tfrac12$, and then moves by $\rnd(\eta r_{t,j})$ notches. Projected Adam is thus a gated discrete update, like \PHDC{}, and the two rules differ only in the statistic that opens the gate. Balles and Hennig interpret $r_t$ as a sign multiplied by a factor that shrinks when gradient noise dominates the gradient mean~\citep{balles2018}. In a stationary approximation where the bias-corrected moments track $\mathbb{E}[g]$ and $\mathbb{E}[g^2]=\mathbb{E}[g]^2+\mathrm{Var}[g]$, and with $S=|\mathbb{E}[g]|/\sqrt{\mathrm{Var}[g]}$ the signal-to-noise ratio of one coordinate,
\begin{equation}
 |r_t|\approx\frac{1}{\sqrt{1+S^{-2}}}<1,
 \qquad\text{so the gate opens when}\qquad
 S\gtrsim\frac{1}{\sqrt{4\eta^2-1}}.
 \label{eq:snrgate}
\end{equation}
For $\eta=3$ the gate requires $S\gtrsim0.17$, for $\eta=1$ it requires $S\gtrsim0.58$, and for $\eta=0.3$ it cannot open at all in the stationary regime, because $|r_t|$ would have to exceed $1.67$. This explains why all three Adam arms selected the largest learning rate on the grid: without a continuous accumulator, a small learning rate does not make updates smaller, it suppresses them, which is the stagnation that Gupta et al.\ attribute to round-to-nearest~\citep{gupta2015}. The boundary selection is therefore a property of projection, and an extended grid would probably favor still larger rates for the same reason.

The controlled validation grid tests this prediction directly. With $\eta_{\mathrm{enc}}=\eta_{\mathrm{cls}}=0.3$, projected Adam with float32 moments changed none of its 5{,}824 stored indices in 180 steps, and its validation accuracy of $9.7\%$ is that of the random initialization. With $\eta_{\mathrm{enc}}=0.3$ alone, the position and value tables stayed frozen while the class table learned, and validation accuracy reached $68.9\%$ to $70.8\%$. With $\eta_{\mathrm{cls}}=0.3$ alone, the class table stayed frozen at random phases while the encoder learned, and accuracy reached $85.8\%$ to $91.1\%$. A small learning rate does not slow projected Adam; it switches it off. The quantized-moment arms at $\eta=0.3$ reached $48.3\%$ and $43.6\%$ instead of chance. Their moment quantization errors occasionally inflate $|r_t|$ past the gate, which is the failure mechanism of \cref{sec:discussion-quantization} acting here in the opposite direction.

The two gates lead to different training dynamics (\cref{fig:dynamics}a). Projected Adam with float32 moments writes to every coordinate at step 1, then its write fraction decays steadily to $0.12\%$ of coordinates per step over steps 91 to 180. Each coordinate changes $11.7\pm0.3$ times during training, and the final training loss is $0.036\pm0.005$. \PHDC{} writes to $72\%$ of coordinates per step between steps 10 and 70. Its write fraction then falls below $5\%$ between steps 69 and 82 across seeds, as the training loss drops from $0.028$ (epoch 8) to $0.009$ (epoch 10), and stays near $0.4\%$ afterwards. Each coordinate changes $55.2\pm3.1$ times, and the final training loss is $0.0081\pm0.0004$, four times lower than that of any Adam arm. The SNR gate of Adam closes when minibatch gradients start to disagree, which happens while the training loss is still comparatively high. The magnitude gate of \PHDC{} closes only when gradients become small, which happens after the training set is nearly fit. The lower training loss does not translate into a significant test advantage, so the controlled result is parity of generalization reached by different routes.

The replays and the grid locate the function of the moments in this setting. On a phase grid without a latent accumulator, the moments mainly decide which coordinates move, and an instantaneous magnitude threshold makes a comparably useful decision without memory. The replacement is not free. The ratio $r_t$ is invariant to the scale of the gradient, whereas $\tau$ is measured in gradient units. \PHDC{} therefore needs a separate threshold for the class table, its selected thresholds differ across datasets by an order of magnitude (\cref{tab:settings}), and by \cref{eq:precision} its gate changes with the bit width. Scale invariance is what the moment arrays buy, and its absence is the main practical cost of $\Gamma=1$. A stateless gate can recover part of this invariance: a threshold set at each step from a quantile of $|g|$ within each table requires transient computation but no persistent state.

\subsection{Failure Mode of Per-Tensor Second-Moment Quantization}
\label{sec:discussion-quantization}
The replays expose a concrete failure of per-tensor affine moment quantization. The quantizer of \cref{app:quantizer} maps a whole tensor onto $2^b-1$ equal intervals between a tracked minimum $\mathrm{lo}$ and maximum $\mathrm{hi}$. A second-moment entry below $\mathrm{lo}+(\mathrm{hi}-\mathrm{lo})/(2(2^b-1))$ is stored at the lowest code and read back as $\mathrm{lo}$, which tracks the tensor minimum and is close to zero. The denominator of $r_t$ then becomes much smaller than the true $\sqrt{v_t}$. The first moment carries its own error of up to half an interval, set by the largest first moments in the tensor, and zero is generally not a representable value. The quotient $r_t$ can therefore be large, $\rnd(\eta r_t)$ becomes a multi-notch jump, and after reduction modulo $K$ the coordinate can land anywhere on the circle.

The measurements follow this chain table by table. In the value table $V$ with 6-bit moments, $31.0\%$ of second-moment entries sit at the lowest code, $7.7\%$ of first-moment entries change sign under quantization, and an average step displaces $22.1\%$ of coordinates by more than three notches. The same rate is $0.009\%$ in the position table $P$ and $0.70\%$ in the class table $\psi$, and $1.7\%$ in $V$ with 8-bit moments. \Cref{fig:dynamics}b shows that these jumps persist through training rather than fading after a transient. As a result, the 6-bit arm still writes to $6.7\%$ of coordinates per step late in training, compared with $0.12\%$ for float32 moments.

The concentration in $V$ follows from how the input distribution reaches each table. The gradient of a value row $V_{\ell,\cdot}$ sums contributions from every position at which value $\ell$ occurs, so its second moment grows with the frequency of $\ell$. In the digit training set, value 0 fills $49.0\%$ of pixel slots and value 16 fills $9.1\%$, while each of the other fifteen values fills between $2.2\%$ and $3.7\%$. Their second moments are small relative to the per-tensor maximum and fall below the first quantization level. \Cref{fig:dynamics}c shows the result: on an average step, the rows of rare values have $10\%$ to $37\%$ of their coordinates displaced by more than three notches, against $0.01\%$ for value 0 and $0.24\%$ for value 16, with a Spearman rank correlation of $-0.98$ between visit frequency and jump rate across the 17 rows. Position rows are immune because every position occurs exactly once in every example, so their second moments are balanced by construction. Class rows receive gradient from every example through the softmax and are nearly balanced.

The damage on digits is partial. The position table, which holds 70\% of the coordinates, is unaffected, the two dominant value rows are stable, and an average step leaves most coordinates of even the rarest rows in place. Test accuracy stays at $91.22\%$ in the controlled comparison, and the screening study shows the same pattern of a modest mean loss with inflated variability: Affine Adam-6 reached $93.15\%$ with a seed SD of $2.83$ points, three times that of \PHDC{}. We have no measurement that attributes this resilience to one of the factors above.

The mechanism makes a testable prediction for the exploratory evaluation. Per-tensor 6-bit moments should fail most severely when the value distribution is dominated by a few bins while many informative bins are rare, and when one shared value table serves many positions. Affine Adam-6 loses most of its accuracy on MNIST, Fashion-MNIST, CIFAR-10, and HAR. The three image datasets have large uniform backgrounds and fit this description; whether the discretized HAR features have a similarly skewed value distribution is not known from the records. The tuning records add that no learning rate rescued Affine Adam-6 on these four datasets (\cref{sec:results-affine}), as expected if the failure lies in the quantizer rather than in the step size. We have not computed the value-bin frequencies of the archived datasets or replayed those runs, and the archived arm also quantizes its parameters and uses a different class head, so the prediction remains untested there. The mechanism also indicates the remedy. Block-wise quantization~\citep{dettmers2022}, or simply one range per value row, gives rare rows their own scale and should remove the failure. The text study offers a second test of the mechanism, this time at 8 bits. In next-symbol prediction the same vocabulary indexes both the value table and the class table, so a skewed symbol distribution affects two tables at once. enwik8 is the most skewed case in this paper: its 256 byte values include frequent ASCII letters and markup alongside many rare bytes. Affine Adam-8 collapses on enwik8 and only there among the text corpora, while on text8, whose 27 symbols are all common, it stays within about one point of Adam-fp32. The three ASCII corpora with 49 to 70 symbols lie in between, with Affine Adam-8 losing 11.7 to 14.3 points against Adam-fp32. This ordering follows the vocabulary size and plausibly its skew, as the mechanism predicts, although we have not measured symbol frequencies or moment codes on these runs. It also shows that 8 bits do not protect per-tensor moments when the frequency range is wide enough.

The large gains of \PHDC{} over Affine Adam-6 and Affine Adam-8 are therefore gains over a specific quantizer construction, and they should not be read as gains over 8-bit adaptive optimization in general.

\subsection{Sources of the Accuracy Gap to Float32 Adam}
\label{sec:discussion-gap}
Adam-fp32 differs from \PHDC{} in three ways at once: it keeps continuous parameters, it keeps moments, and it has a free complex class head. The screening systems separate part of this bundle on digits (\cref{tab:screening}). A 6-bit integer encoder trained with \cref{eq:update}, in place of the continuous encoder of Adam-fp32 and with the free float32 head and its Adam moments kept, lowers accuracy from $96.48\%$ to $94.33\%$, a loss of $2.15$ points. A head constrained to unit-modulus 6-bit phases and stripped of its moments costs a further $0.63$ points, to $93.70\%$. Most of the gap on digits therefore arises in the encoder, where continuous phases moved by Adam are replaced by grid phases moved one notch at a time. The size of the gap also depends on the task. It averages $7.0$ points on the six image and tabular datasets, where ISOLET alone contributes $15.6$, and $3.3$ points on the five text corpora, where every corpus uses the same 64 context positions and the same dimension.

The controlled comparison indicates which part of that replacement matters. When Adam itself is restricted to the grid and loses its continuous copy, it no longer outperforms \PHDC{}: projected Adam with float32 moments reaches $90.78\%$ against $92.44\%$. The two digit studies differ in dimension ($D=64$ against $D=96$) and selection protocol, so their numbers cannot be subtracted. Their pattern is nevertheless consistent: the advantage of Adam-fp32 comes mainly from its continuous parameters, which can accumulate many small updates between grid points, and much less from its moments. In the language of \cref{sec:bg-state}, the master copy, not the optimizer's history of gradients, is the expensive state that buys accuracy.

At four bits the head becomes the bottleneck. The hybrid system with a 4-bit integer encoder reaches $94.89\%$, so one-notch training of a 4-bit encoder works on digits. \PHDC{} with a 4-bit class table falls to $84.81\%$. The class table enters every score directly: one notch on $\psi_{c,d}$ changes $z_c$ by $(e^{-ia\sigma}-1)\,h_d\overline{w_{c,d}}$, whose magnitude $2\sin(a/2)|h_d|$ is $39\%$ of the coordinate's whole contribution at $k=4$. A unit-modulus head also cannot weight coordinates by importance, which a free head does through its magnitudes. Both properties make a coarse phase-only head a blunt instrument, and they suggest a hybrid design for very low widths: a phase encoder trained without state and a small head with a few more bits.

\subsection{Phase Width as a Joint Control of Capacity and Optimization}
\label{sec:discussion-width}
\Cref{eq:precision} identifies two ways in which the bit width enters the optimizer: the step size $a$ and the angular gate $\tau/a$. The phase memory adds a third effect through its shared value table. One notch on $V_{\ell,d}$ rotates every bound phasor $u_{p,d}$ with $x_p=\ell$ at once. For a step $\sigma\in\{-1,1\}$,
\begin{equation}
 \Delta h_d=\bigl(e^{i a\sigma}-1\bigr)\sum_{p:\,x_p=\ell}u_{p,d},
 \qquad
 |\Delta h_d|\leq2\sin(a/2)\,n_\ell(x),
 \label{eq:sharedstep}
\end{equation}
where $n_\ell(x)$ counts the positions of $x$ that hold value $\ell$. The factor $2\sin(a/2)$ equals $0.390$ at $k=4$, $0.098$ at $k=6$, and $0.025$ at $k=8$. At four bits, one notch of a frequent value row can therefore move the bundle of an image with $F=784$ positions by up to about 300 units in a coordinate whose magnitude cannot exceed 784. The same step on Letter, with $F=16$, moves the bundle by at most about 6 units. This predicts that four-bit training is unstable when inputs have many positions, and the sweep agrees: Letter, the only dataset with few positions, reaches $69.63\%$ at $k=4$, whereas the five datasets with 561 to 1{,}024 positions remain between $14.67\%$ and $38.36\%$, with SDs up to $15.20$ points.

The same geometry limits the justification of the update at coarse widths. In index units, the curvature constant of \cref{eq:descent} scales as $\Lambda=a^2\Lambda_\theta$, where $\Lambda_\theta$ is the curvature in angle. With half as many notches, $a$ doubles, which doubles the first-order gain of a step and quadruples its second-order term, so the ratio of curvature penalty to predicted gain grows in proportion to $a$ and a one-notch move stops being a local move. At $k=2$, a step is a quarter turn and the linear model behind \cref{prop:local} carries little information about the actual change in loss, which is consistent with accuracies near the class-frequency level. The screening records support an optimization rather than a capacity explanation. A dense 2-bit table reached $96\%$ on digits in the generated-memory study (\cref{tab:generated}), so four phase levels can represent a good digit classifier. One-notch training at $k=2$ on the same task reached only $8.3\%$ to $10.0\%$ across three thresholds, near the $10\%$ chance level, with nearly every eligible coordinate changing at every step.

At $k=8$ the opposite constraints bind. Steps are small, but travel is not the limiting factor: Letter receives 1{,}375 updates, enough for a coordinate to traverse more than five full turns at $1.4^\circ$ per step. The angular gate, however, is four times stricter than at $k=6$ for the same index-space threshold, and the selected encoder threshold lies at the smallest grid value on five of six datasets. The accuracy losses at eight bits on the tabular datasets are therefore more plausibly an under-tuned gate than a property of the representation. The digit screening already demonstrates this effect: extending the threshold grid at $k=8$ raised accuracy from $86.22\%$ to $93.81\%$, level with $k=6$. A sweep that holds the angular threshold $\tau_\theta$ fixed across widths and extends the grid below its current minimum would settle the question on the six datasets.

\subsection{Logical State, Write Traffic, and Fixed-Point Arithmetic}
\label{sec:discussion-hardware}
The storage results concern logical state, and their practical value depends on how a training system is built. The most direct beneficiary is a device that must keep its learned tables resident while it trains through a small workspace, for example an embedded system that adapts a classifier to its user. In that setting the persistent arrays, not the transient computation, set the minimum memory, and $\Gamma=1$ means that a device able to store the model can also store everything it needs to keep training it.

Statelessness also changes the write pattern. \PHDC{} writes each parameter $55.2$ times over training, against $11.7$ to $24.0$ times for the projected Adam arms, because its gate stays open longer. Adam, however, rewrites both moment arrays at every step. Per coordinate and step, \PHDC{} writes $1.84$ bits on average, against $12.8$, $16.4$, and $64.4$ bits for projected Adam with 6-bit, 8-bit, and float32 moments. For memories in which writes dominate energy or wear, such as nonvolatile in-memory computing arrays, total write traffic favors the stateless rule while per-cell write counts on the parameter array favor Adam. Which matters depends on the device.

A practical implementation must also bound transient memory. Because \cref{eq:update} needs only $g_j$ to update $q_j$, gradients can be computed and applied row by row without materializing a full gradient array, and phases can be decoded through a lookup table of $K$ cosine and sine values.

The fixed-point screening addresses the arithmetic. The update consumes only the sign of each gradient and a comparison with the threshold, so it should tolerate rounding errors that perturb magnitudes without systematically reversing signs. The results fit this view. At one frozen state, 4-bit softmax arithmetic reverses about a third of all gradient signs, yet training with that arithmetic ends within one standard deviation of the float control. Two properties of the rule plausibly explain the difference. Rounding is most likely to reverse the sign of a gradient whose exact value is close to zero, and such a coordinate is usually below the threshold and does not move. A sign error that does pass the gate moves a coordinate by one notch, and if the errors are not biased in one direction, later steps undo it, as minibatch noise is undone in ordinary stochastic training. We did not measure the magnitude of the reversed gradients, so this explanation is a hypothesis. The simulation also shows where precision is needed: the head's softmax and logit arithmetic, not the trigonometric table, sets the sign agreement (\cref{tab:fpsnapshot}). An integer implementation would still have to fix accumulator widths and handle overflow and saturation, which the simulation does not model, as WAGE does for deep networks~\citep{wu2018}. The main experiments store codes in floating-point arrays and compute in floating point, so they demonstrate none of the memory savings physically.

\section{Limitations and Future Work}
\label{sec:limitations}

\subsection{Statistical Strength of the Controlled Evidence}
The central comparison uses one small dataset, one stratified example-level split, and five seeds, and its paired interval includes zero. It supports parity, not superiority. The screening studies, the six-dataset evaluation, and the text study extend the picture to stronger competitors and other tasks, but all three selected hyperparameters on the data they report, the text study searched only four configurations per system, and the screening code is not released. All Adam arms selected the largest learning rate on the grid, and \cref{eq:snrgate} predicts that larger rates could help them further. The six-dataset evaluation is broader but reuses evaluation labels for selection and compares different class heads, so it cannot isolate the optimizer~\citep{cawley2010}. A decisive study would hold the phase grid, head, initialization, and data partitions fixed across several datasets with held-out selection and writer- or speaker-disjoint splits. It would extend the Adam learning-rate grid upward and add three competitors that the present results single out: block-wise 8-bit Adam~\citep{dettmers2022}, Adam with one quantization range per table row, and Adafactor~\citep{shazeer2018}. For text, it would also report the accuracy of always predicting the most frequent symbol, a stronger reference than uniform guessing, and measure symbol frequencies to test the quantization mechanism directly. The instrumentation of \cref{sec:setup-diagnostics} should accompany such a study so that every accuracy difference comes with its mechanism.

\subsection{Logical State without a Packed Training Implementation}
All memory ratios describe ideal packed parameter and moment arrays. The code simulates low-bit codes in floating-point arrays and does not measure peak memory, latency, or energy. The fixed-point screening rounds intermediate values but simulates neither integer overflow nor saturation, and it ran on one small dataset whose code is not released. The paper therefore establishes an accounting opportunity, a faithful simulation of the update, and screening evidence that rounded gradients suffice, not an end-to-end integer system. The natural next step is a packed implementation on a memory-constrained target, such as a microcontroller, that streams row-wise gradients as described in \cref{sec:discussion-hardware} and reports measured peak memory, energy per update, and per-cell write counts alongside accuracy.

\subsection{Scale-Dependent Threshold and Restricted Model Family}
The threshold of \PHDC{} lacks the scale invariance of Adam's moment ratio, so it must be tuned per table, per dataset, and per bit width, and the phase-width sweep shows that a poorly matched gate can cost many points of accuracy. The experiments also cover only shallow item-memory classifiers with one magnitude readout. They do not test whether the learned vectors keep the unbinding and similarity properties that make random HDC codes compositional, and they do not compare against deep models at a matched resource budget. Two extensions address these limits without adding persistent state: a gate computed at each step from a quantile of the current gradient magnitudes in each table, and the application of moment-free one-notch updates to other discrete lookup structures, such as quantized embedding tables, together with direct tests of algebraic fidelity on the learned codes~\citep{snyder2026}.

%% file: appendix.tex
\section{Notation}
\label{app:notation}
\begin{table}[h]
\centering\small
\caption{Symbols used in the paper. Each symbol has one meaning throughout.}
\label{tab:notation}
\begin{tabularx}{\textwidth}{lX}
\toprule
Symbol & Meaning\\
\midrule
$x=(x_1,\ldots,x_F)$, $y$ & discretized input with values $x_p\in\{0,\ldots,B-1\}$; class label\\
$F$, $B$, $C$, $D$ & number of positions, value bins, classes; representation dimension\\
$p$, $\ell$, $c$, $d$, $j$ & indices of position, value bin, class, dimension, and concatenated coordinate\\
$k$, $K=2^k$, $a=2\pi/K$ & phase bit width, number of phase levels, angle of one notch\\
$\ZK$ & integers $\{0,\ldots,K-1\}$ with addition modulo $K$\\
$P$, $V$, $\psi$ & position, value, and class phase tables\\
$u_{p,d}$, $h_d$ & bound phasor of position $p$; bundled representation at coordinate $d$\\
$w_{c,d}$, $z_c$, $s_c$, $\hat y_c$ & class phasor, complex overlap, magnitude score, softmax probability\\
$T$, $\eps$ & softmax temperature; numerical stabilizer in $s_c$\\
$\mathcal{B}$, $\loss$ & minibatch; mean cross-entropy loss\\
$q\in\ZK^N$, $N$ & all phase indices concatenated; their number $D(F+B+C)$\\
$g$, $\delta$, $\tau_j$ & minibatch gradient in index units; integer move; write cost (threshold)\\
$\tau_{\mathrm{enc}}$, $\tau_{\mathrm{cls}}$ & thresholds for $(P,V)$ and for $\psi$\\
$\mathcal{A}$, $\Lambda$, $\Lambda_\theta$ & active coordinates; gradient Lipschitz constant in index and angle units\\
$R$ & write budget of the budgeted variant\\
$\theta=aq$, $g_\theta$ & coordinate in radians; gradient with respect to it\\
$\Gamma$ & state overhead (\cref{def:overhead})\\
$M_{\mathrm{phase}}$, $M_{\mathrm{proj}}$, $M_{\mathrm{Adam},b}$ & persistent state of \PHDC{}, projected Adam, archived Adam arms\\
$b$, $N_E$, $N_C$ & moment bit width; encoder and class coordinate counts\\
$m_t$, $v_t$, $r_t$ & first moment, second moment, bias-corrected moment ratio\\
$\beta_1$, $\beta_2$, $\eps_{\mathrm{A}}$, $\eta$ & Adam decay rates, stabilizer, learning rate ($\eta_{\mathrm{enc}}$, $\eta_{\mathrm{cls}}$)\\
$S$ & per-coordinate gradient signal-to-noise ratio\\
$\sigma$, $n_\ell(x)$ & sign of a one-notch move; number of positions of $x$ holding value $\ell$\\
$n$ & number of seeds in a sample standard deviation\\
$\xi$, $\xi^{\min}_p$, $\xi^{\max}_p$ & raw feature value; training minimum and maximum of feature $p$\\
$\mu$, $\mathrm{lo}$, $\mathrm{hi}$, $\rho$ & tensor being quantized; tracked range; range-tracking rate\\
\bottomrule
\end{tabularx}
\end{table}

\FloatBarrier
\section{Implementation and Protocol Details}
\label{app:protocol}

\subsection{Input Discretization}
For a tabular feature $p$ with training minimum $\xi^{\min}_p$ and maximum $\xi^{\max}_p$ (replaced by $\xi^{\min}_p+10^{-6}$ if the range is zero), a raw value $\xi$ is mapped to
\begin{equation}
 x_p=\clip\!\left(\left\lfloor\frac{\xi-\xi^{\min}_p}{\xi^{\max}_p-\xi^{\min}_p}\,(B-1)\right\rfloor,\,0,\,B-1\right).
\end{equation}
The same map is applied to evaluation examples. With this formula, the top bin is reached mainly by values at the fitted maximum; reproductions should not replace it by a histogram with $B$ equal intervals. MNIST and Fashion-MNIST map an intensity $\xi\in\{0,\ldots,255\}$ to $\lfloor16\xi/256\rfloor$. CIFAR-10 quantizes each color channel to four levels and combines them as $x_p=16x_p^{\mathrm{R}}+4x_p^{\mathrm{G}}+x_p^{\mathrm{B}}$, which gives 64 joint color values at each of 1{,}024 positions. No convolution or spatial augmentation is used.

\subsection{Exploratory Training Configurations}
For each Adam arm the encoder learning rate ranges over $\{0.1,0.3,0.5,1.0\}$ and the class learning rate over $\{0.03,0.1\}$. For \PHDC{} both thresholds range over $\{3\times10^{-5},10^{-4},3\times10^{-4}\}$. Adam uses $(\beta_1,\beta_2)=(0.9,0.999)$ and $\eps_{\mathrm{A}}=10^{-8}$, without weight decay or schedule. \Cref{tab:settings} lists the selected pairs. The phase-width sweep reuses the threshold grid at every width, and its six-bit column reuses the main records rather than repeating them.

\begin{table}[h]
\centering\small
\caption{Selected exploratory configurations as (encoder, class) pairs: learning rates for Adam, thresholds in phase-index units for \PHDC{}. Selection used the evaluation set on seed 0.}
\label{tab:settings}
\begin{tabular}{lrrrr}
\toprule
Dataset & Adam-fp32 & Affine Adam-8 & Affine Adam-6 & Phase-HDC-6\\
\midrule
\input{tables/gpu_settings}
\end{tabular}
\end{table}

In the exploratory scripts, the position and value tables of every arm, including \PHDC{}, are initialized as real numbers drawn uniformly from $[0,K)$, so the first gradient of \PHDC{} is evaluated before the first projection onto integers; every later state lies on the grid. The controlled implementation initializes all three tables with integer indices, which matches \cref{alg:phasehdc} from the first step. Adam-fp32 phases are reduced modulo $K$ after each update but not rounded, and the affine-quantized Adam phases follow their adaptive tensor ranges without explicit modular reduction. These differences are part of the recorded systems and are reported rather than corrected after the fact.

\subsection{Affine Moment Quantization}
\label{app:quantizer}
For a tensor $\mu$ and bit width $b$, the simulated quantizer stores integer codes
\begin{equation}
 Q(\mu)=\clip\!\left(\rnd\!\left[\frac{\mu-\mathrm{lo}}{\mathrm{hi}-\mathrm{lo}}\,(2^b-1)\right],\,0,\,2^b-1\right),
 \qquad
 \widehat{\mu}=\mathrm{lo}+\frac{\mathrm{hi}-\mathrm{lo}}{2^b-1}\,Q(\mu).
\end{equation}
The first write sets $(\mathrm{lo},\mathrm{hi})$ to the minimum and maximum of $\mu$. Later writes track the range by
\begin{equation}
 \mathrm{lo}\gets(1-\rho)\,\mathrm{lo}+\rho\min\mu,\qquad
 \mathrm{hi}\gets(1-\rho)\,\mathrm{hi}+\rho\max\mu,\qquad\rho=0.01,
\end{equation}
and a degenerate range is widened by $10^{-6}$. In projected Adam, the stored moments are dequantized, updated with the current gradient by \cref{eq:projadam}, used to form the candidate step, and requantized for the next minibatch. The parameters remain exact grid indices after rounding and modular reduction. The float32 arm follows the same sequence without moment quantization. The scheme uses neither stochastic rounding, block-wise ranges, nor a master-weight accumulator.

Two float32 range endpoints per quantized tensor would add 48 bytes for the six moment tensors of the controlled experiment and 96 bytes for the twelve quantized tensors of the archived arms. All tables omit these terms consistently, together with step counters and hyperparameters, so the reported figures are array-state accounting rather than whole-program storage.

\subsection{Controlled Protocol and Numerical Checks}
The controlled split uses random seed 20261005 for the 80/20 partition into a pool and a test set, and seed 20261006 for the 75/25 partition of the pool into training and validation sets, stratified by class at both stages. The released package contains the split indices, a hash of the dataset content, the full grid, the selected configurations, library versions, and all final checkpoints. The gradient check uses float64 central differences with displacement $10^{-5}$. A second check enumerates all $3^4=81$ moves on 100 random four-coordinate instances of \cref{eq:local} and confirms that \cref{eq:update} attains the minimum in each. A third check verifies wraparound at indices 0 and 63.

A first execution of the controlled experiment stored simulated quantized moment codes in float64 containers instead of float32. It was corrected for consistency across arms and rerun with the same split, grids, seeds, and budget; no grid was changed in response to test outcomes. All reported controlled values come from the corrected execution, and the instrumented replays of \cref{sec:setup-diagnostics} reproduce its checkpoints exactly.

\clearpage
\section{Additional Results}
\label{app:width}
\begin{table}[h]
\centering\small
\caption{Phase-width sweep of \PHDC{}: mean $\pm$ sample SD of exploratory accuracy (\%) over three seeds. The six-bit column is the main exploratory record. All widths share the threshold grid and update count; the angular step differs. Bold marks the best width per dataset.}
\label{tab:width}
\begin{tabular}{lrrrr}
\toprule
Dataset & $k=2$ & $k=4$ & $k=6$ & $k=8$\\
\midrule
\input{tables/phase_width}
\end{tabular}
\end{table}

\begin{table}[h]
\centering\small
\caption{Test accuracy (\%) of every controlled run. All seeds share the fixed data split; the projected Adam arms share the 6-bit parameters of \PHDC{} and differ only in moment precision. Bold marks the best method per seed.}
\label{tab:seeds}
\begin{tabular}{rrrrr}
\toprule
Seed & Phase-HDC & Adam, 6-bit moments & Adam, 8-bit moments & Adam, float32 moments\\
\midrule
\input{tables/cpu_seeds}
\end{tabular}
\end{table}

\FloatBarrier
\section{Hyperparameter Sensitivity}
\label{app:sensitivity}
\Cref{tab:sensitivity} reports, for every system of the six-dataset evaluation, the lowest and highest accuracy reached across its tuning grid on seed 0. The ranges show how much each system depends on its hyperparameters, and they show that the collapse of Affine Adam-6 on HAR, MNIST, Fashion-MNIST, and CIFAR-10 persists across every learning rate tried. \PHDC{} is not insensitive: on MNIST its thresholds span $45.0\%$ to $91.7\%$. In the controlled comparison, its nine configurations span $85.3\%$ to $94.7\%$ validation accuracy, whereas the projected Adam arms span $9.7\%$ to $92.2\%$ (float32 moments), $48.3\%$ to $94.7\%$ (6-bit), and $43.6\%$ to $93.6\%$ (8-bit), because their smallest learning rates freeze the parameters (\cref{sec:discussion-gating}).

\begin{table}[h]
\centering\small
\caption{Range of tuning-seed accuracy (\%) across each system's hyperparameter grid on six datasets (eight learning-rate pairs for Adam, nine threshold pairs for \PHDC{}).}
\label{tab:sensitivity}
\begin{tabular}{lrrrr}
\toprule
Dataset & Adam-fp32 & Affine Adam-8 & Affine Adam-6 & Phase-HDC-6\\
\midrule
\input{tables/gpu_sensitivity}
\end{tabular}
\end{table}

For the text study, \cref{tab:textsensitivity} gives the same ranges over its four-configuration grids. The collapse of Affine Adam-8 on enwik8 persists across all four learning-rate pairs, whose best reached $12.6\%$ on the tuning seed.

\begin{table}[h]
\centering\small
\caption{Range of tuning-seed accuracy (\%) across each system's four-configuration grid on five text corpora.}
\label{tab:textsensitivity}
\begin{tabular}{lrrr}
\toprule
Corpus & Adam-fp32 & Affine Adam-8 & Phase-HDC-6\\
\midrule
\input{tables/text_sensitivity}
\end{tabular}
\end{table}

The text study used learning-rate pairs $(\eta_{\mathrm{enc}},\eta_{\mathrm{cls}})\in\{(0.3,0.1),(0.5,0.03),(0.5,0.1),(1.0,0.03)\}$ for both Adam systems and thresholds $\tau_{\mathrm{enc}},\tau_{\mathrm{cls}}\in\{10^{-4},3\times10^{-4}\}$ for \PHDC{}. Its windows are sampled without replacement from the training and test segments with a fixed sampling seed, so all systems and seeds see the same windows. Tiny Shakespeare is read from the char-rnn repository, TinyStories from its public training split, PTB from the language-modeling training file distributed with a public LSTM implementation, and text8 and enwik8 from the benchmark archive; each source URL is recorded in the released script.

\FloatBarrier
\section{Provenance and Accounting of the Screening Studies}
\label{app:screening-accounting}
The screening studies of \cref{sec:setup-screening} are reported from the archived reports and summary tables of the project, which record means and standard deviations over five seeds but neither per-seed values nor the executable code. They are therefore not regenerated by the released scripts, and their selection protocol is treated as exploratory.

State figures in \cref{tab:screening} are recomputed from the architecture ($D=96$, $F=64$, $B=17$, $C=10$) with \cref{eq:memoryarchive}, which counts the free complex class head of the Adam systems as $2CD$ real coordinates. One archived report counts that head as $CD$ coordinates for Adam-fp32 and Affine Adam-8, which gives 102.4~KB and 25.6~KB instead of 113.6~KiB and 28.4~KiB, and lists 20.8~KB for Affine Adam-6 instead of 21.3~KiB. Under that accounting the saving of \PHDC{} at $k=6$ against Adam-fp32 is $16.0\times$ rather than $17.8\times$. We use the uniform accounting throughout so that ratios are comparable across studies. The accuracy reference of Adam-fp32 is the run paired with $k=6$, $96.48\pm0.56\%$; a repeat paired with $k=4$, selecting a different learning rate, reached $96.33\pm0.66\%$.

A comparison run before the matched protocol was adopted is superseded and excluded. It tuned three thresholds for \PHDC{} but used a single untuned learning rate for Adam, and it reported \PHDC{} at $k=4$ ahead of Adam ($95.4\%$ against $88.3\%$). Under equal-size grids, Adam-fp32 led every integer system instead. The episode illustrates the selection effects discussed by Cawley and Talbot~\citep{cawley2010} and motivated the matched grids used in every study reported here. The threshold sweep of that run at $k=2$ (accuracies of $8.3\%$ to $10.0\%$) is a property of \PHDC{} alone and is cited in \cref{sec:discussion-width}.

The generated-memory and fixed-point studies of \cref{sec:results-generated,sec:results-fixedpoint} followed the same screening setup. In the fixed-point study, an early version clipped logits to a range that excluded about $10\%$ of examples; the reported results use the corrected range, with which the 20-bit simulation matched the float reference to within $4\times10^{-8}$.

%% file: tables/gpu_settings.tex
Letter & $(0.3,\ 0.1)$ & $(0.3,\ 0.1)$ & $(0.5,\ 0.1)$ & $(10^{-4},\ 10^{-4})$ \\
ISOLET & $(1,\ 0.03)$ & $(1,\ 0.03)$ & $(1,\ 0.03)$ & $(3\times10^{-5},\ 10^{-4})$ \\
HAR & $(1,\ 0.03)$ & $(0.5,\ 0.03)$ & $(1,\ 0.1)$ & $(3\times10^{-5},\ 3\times10^{-5})$ \\
MNIST & $(0.1,\ 0.03)$ & $(0.5,\ 0.1)$ & $(1,\ 0.1)$ & $(3\times10^{-5},\ 3\times10^{-4})$ \\
Fashion-MNIST & $(0.3,\ 0.03)$ & $(0.3,\ 0.03)$ & $(1,\ 0.1)$ & $(3\times10^{-5},\ 3\times10^{-4})$ \\
CIFAR-10 & $(0.5,\ 0.03)$ & $(1,\ 0.1)$ & $(0.3,\ 0.03)$ & $(3\times10^{-4},\ 3\times10^{-4})$ \\
\bottomrule

%% file: tables/phase_width.tex
Letter & $3.75 \pm 0.16$ & $69.63 \pm 2.63$ & $\boldsymbol{79.68 \pm 0.25}$ & $77.24 \pm 0.19$ \\
ISOLET & $4.33 \pm 0.34$ & $34.27 \pm 15.20$ & $\boldsymbol{77.49 \pm 0.75}$ & $57.72 \pm 1.87$ \\
HAR & $18.23 \pm 0.02$ & $38.36 \pm 3.79$ & $\boldsymbol{87.12 \pm 4.74}$ & $84.38 \pm 0.71$ \\
MNIST & $10.02 \pm 1.23$ & $16.06 \pm 11.65$ & $91.22 \pm 1.18$ & $\boldsymbol{92.05 \pm 0.77}$ \\
Fashion-MNIST & $9.95 \pm 1.03$ & $17.90 \pm 9.61$ & $80.81 \pm 1.73$ & $\boldsymbol{81.97 \pm 0.31}$ \\
CIFAR-10 & $11.99 \pm 0.72$ & $14.67 \pm 3.57$ & $32.54 \pm 2.14$ & $\boldsymbol{34.02 \pm 1.16}$ \\
\bottomrule

%% file: tables/cpu_seeds.tex
1 & \textbf{92.78} & 90.00 & 91.39 & 92.22 \\
2 & \textbf{93.89} & 92.22 & 91.67 & 93.06 \\
3 & \textbf{92.22} & 91.39 & 90.56 & 87.78 \\
4 & 91.67 & \textbf{92.50} & 90.83 & 90.00 \\
5 & 91.67 & 90.00 & \textbf{92.78} & 90.83 \\
\midrule
Mean & \textbf{92.44} & 91.22 & 91.44 & 90.78 \\
\bottomrule

%% file: tables/gpu_sensitivity.tex
Letter & 86.8--89.3 & 77.7--86.6 & 70.2--86.1 & 64.2--80.5 \\
ISOLET & 77.1--93.5 & 62.9--92.5 & 51.8--88.6 & 54.7--78.4 \\
HAR & 90.3--94.2 & 69.8--89.7 & 13.8--72.0 & 75.5--90.0 \\
MNIST & 85.8--93.5 & 48.4--93.2 & 9.2--72.9 & 45.0--91.7 \\
Fashion-MNIST & 74.3--84.0 & 76.4--84.6 & 14.2--60.7 & 73.1--83.8 \\
CIFAR-10 & 23.5--38.7 & 10.1--28.1 & 10.0--13.9 & 21.7--31.7 \\
\bottomrule

%% file: tables/text_sensitivity.tex
Tiny Shakespeare & 23.6--27.7 & 7.8--16.7 & 18.1--24.4 \\
TinyStories & 32.1--35.6 & 6.0--22.3 & 26.4--32.2 \\
PTB & 28.1--31.7 & 10.6--23.1 & 24.1--29.7 \\
text8 & 25.7--27.8 & 24.3--27.4 & 23.7--26.0 \\
enwik8 & 22.4--25.1 & 0.0--12.6 & 14.4--21.0 \\
\bottomrule